\documentclass{article}

\PassOptionsToPackage{numbers, compress}{natbib}

\usepackage[final]{neurips_2019}

\usepackage[utf8]{inputenc} 
\usepackage[T1]{fontenc}    
\usepackage{hyperref}       
\usepackage{url}            
\usepackage{booktabs}       
\usepackage{amsfonts}       
\usepackage{nicefrac}       
\usepackage{microtype}      
\usepackage{amsthm}
\usepackage{amssymb}
\usepackage{amsmath}
\usepackage{bm}
\usepackage{xcolor}
\usepackage{tikz}
\usetikzlibrary{shapes,arrows,patterns}
\usepackage{xspace}
\usepackage{adjustbox}
\usepackage{algorithm}
\usepackage{algorithmic}
\usepackage{dsfont}
\usepackage{wrapfig}
\usepackage{float}
\usepackage{enumitem}
\usepackage{setspace}
\usepackage[stable]{footmisc}

\newtheorem{theorem}{Theorem}[subsection]
\newtheorem{proposition}{Proposition}[subsection]
\newtheorem{corollary}{Corollary}[subsection]

\def\mh{{\hbox{-}}}

\newcommand{\combo}{$\mathrm{COMBO}$\xspace}
\newcommand{\GP}{GP\xspace}
\newcommand{\BO}{BO\xspace}
\newcommand{\ARD}{ARD\xspace}

\DeclareMathOperator*{\argmax}{arg\; max}
\DeclareMathOperator*{\argmin}{arg\; min}
\DeclareMathOperator{\x}{\mathbf{x}}
\DeclareMathOperator{\G}{\mathcal{G}}
\DeclareMathOperator{\D}{\mathcal{D}}
\DeclareMathOperator{\V}{\mathcal{V}}
\DeclareMathOperator{\E}{\mathcal{E}}

\newcommand{\ie}{\emph{i.e.}}
\newcommand{\eg}{\emph{e.g.}}

\title{Combinatorial Bayesian Optimization\\using the Graph Cartesian Product}

\author{
    Changyong Oh\textsuperscript{1} 
    Jakub M. Tomczak\textsuperscript{2} 
    Efstratios Gavves\textsuperscript{1} 
    Max Welling\textsuperscript{1,2,3}\\
    \textsuperscript{1} University of Amsterdam 
    \textsuperscript{2} Qualcomm AI Research 
    \textsuperscript{3} CIFAR\\
    \texttt{C.Oh@uva.nl, jtomczak@qti.qualcomm.com, egavves@uva.nl, m.welling@uva.nl}
}

\begin{document}

\maketitle

\begin{abstract}
This paper focuses on Bayesian Optimization (\BO) for objectives on combinatorial search spaces, including ordinal and categorical variables.
Despite the abundance of potential applications of Combinatorial \BO, including chipset configuration search and neural architecture search, only a handful of methods have been proposed.
We introduce \combo, a new Gaussian Process (\GP) \BO. 
\combo quantifies ``smoothness'' of functions on combinatorial search spaces by utilizing a \emph{combinatorial graph}.
The vertex set of the combinatorial graph consists of all possible joint assignments of the variables, while edges are constructed using the graph Cartesian product of the sub-graphs that represent the individual variables. 
On this combinatorial graph, we propose an ARD diffusion kernel with which the \GP is able to model high-order interactions between variables leading to better performance.
Moreover, using the Horseshoe prior for the scale parameter in the \ARD diffusion kernel results in an effective variable selection procedure, making \combo suitable for high dimensional problems.
Computationally, in \combo the graph Cartesian product allows the Graph Fourier Transform calculation to scale linearly instead of exponentially.
We validate \combo in a wide array of realistic benchmarks, including weighted maximum satisfiability problems and neural architecture search.
\combo outperforms consistently the latest state-of-the-art while maintaining computational and statistical efficiency.
\end{abstract}

\section{Introduction}
\label{sect:introduction}
This paper focuses on Bayesian Optimization (\BO)~\citep{snoek2012practical} for objectives on combinatorial search spaces consisting of ordinal or categorical variables.
%
%
Combinatorial BO~\citep{jones1998efficient} has many applications, including finding optimal chipset configurations, discovering the optimal architecture of a deep neural network or the optimization of compilers to embed software on hardware optimally.
%
%
All these applications, where Combinatorial BO is potentially useful, feature the following properties.
They \emph{(i)} have black-box objectives for which gradient-based optimizers~\citep{wilson2014using} are not amenable, \emph{(ii)} have expensive evaluation procedures for which methods with low sample efficiency such as, evolution ~\citep{freitas2009review} or genetic~\citep{davis1991handbook} algorithms are unsuitable, and \emph{(iii)} have noisy evaluations and highly non-linear objectives for which simple and exact solutions are inaccurate~\citep{brochu2010tutorial, frazier2018tutorial, shahriari2015taking}.

Interestingly, most BO methods in the literature have focused on continuous~\citep{movckus1975bayesian} rather than combinatorial search spaces.
One of the reasons is that the most successful BO methods are built on top of Gaussian Processes ({\GP}s) \citep{kandasamy2016gaussian, oh2018bock, snoek2012practical}.
As GPs rely on the smoothness defined by a kernel to model uncertainty~\citep{rasmussen2006gaussian}, they are originally proposed for, and mostly used in, continuous input spaces.
In spite of the presence of kernels proposed on combinatorial structures ~\citep{haussler1999convolution, kondor2002diffusion, smola2003kernels},
%
%
to date the relation between the smoothness of graph signals and the smoothness of functions defined on combinatorial structures has been overlooked and not been exploited for BO on combinatorial structures.
A simple solution is to use continuous kernels and round them up.
This rounding, however, is not incorporated when computing the covariances at the next BO iteration~\citep{garrido2018dealing}, leading to unwanted artifacts.
Furthermore, when considering combinatorial search spaces the number of possible configurations quickly explodes: for $M$ categorical variables with $k$ categories the number of possible combinations scales with $\mathcal{O}(k^M)$.
%
%
Applying BO with GPs on combinatorial spaces is, therefore, not straightforward.

We propose \combo, a novel Combinatorial BO designed to tackle the aforementioned problems of lack of smoothness and computational complexity on combinatorial structures.
To introduce smoothness of a function on combinatorial structures, we propose the \emph{combinatorial graph}.
The combinatorial graph comprises sub-graphs --one per categorical (or ordinal) variable-- later combined by the graph Cartesian product.
The combinatorial graph contains as vertices all possible combinatorial choices.
We define then smoothness of functions on combinatorial structures to be the smoothness of graph signals using the Graph Fourier Transform (GFT)~\citep{ortega2018graph}.
Specifically, we propose as our \GP kernel on the graph a variant of the diffusion kernel, the automatic relevance determination(ARD) diffusion kernel, for which computing the GFT is computationally tractable via a decomposition of the eigensystem.
With a \GP on a graph \combo accounts for arbitrarily high order interactions between variables.
Moreover, using the sparsity-inducing Horseshoe prior~\citep{carvalho2009handling} on the ARD parameters \combo performs variable selection and scales up to high-dimensional.
\combo allows for accurate, efficient and large-scale BO on combinatorial search spaces.




In this work, we make the following contributions.
First, we show how to introduce smoothness on combinatorial search spaces by introducing combinatorial graphs.
On top of a combinatorial graph we define a kernel using the GFT.
Second, we present an algorithm for Combinatorial BO that is computationally scalable to high dimensional problems.
Third, we introduce individual scale parameters for each variable making the diffusion kernel more flexible.
When adopting a sparsity inducing Horseshoe prior~\citep{carvalho2009handling, carvalho2010horseshoe}, \combo performs variable selection which makes it scalable to high dimensional problems.
We validate \combo extensively on \emph{(i)} four numerical benchmarks, as well as two realistic test cases: \emph{(ii)} the weighted maximum satisfiability problem \citep{hansen1990algorithms, resende1997approximate}, where one must find boolean values that maximize the combined weights of satisfied clauses, that can be made true by turning on and off the variables in the formula, \emph{(iii)} neural architecture search~\citep{elsken2018neural, wistuba2019survey}.
Results show that \combo consistently outperforms all competitors.

\section{Method}
\subsection{Bayesian optimization with Gaussian processes}
\vspace{-4pt}

Bayesian optimization (\BO) aims at finding the global optimum of a black-box function $f$ over a search space $\mathcal{X}$, namely, $\x_{opt}=\argmin_{\x \in \mathcal{X}} f(\x)$. At each round, a surrogate model attempts to approximate $f(\x)$ based on the evaluations so far, $\D = \{(\x_i, y_i=f(\x_i))\}$.
Then an acquisition function suggests the most promising point $\x_{i+1}$ that should be evaluated.
The $\D$ is appended by the new evaluation, $\D=\D\cup(\{\x_{i+1}, y_{i+1})\}$. 
The process repeats until the evaluation budget is depleted. 

The crucial design choice in \BO is the surrogate model that models $f(\cdot)$ in terms of \emph{(i)} a predictive mean to predict $f(\cdot)$, and \emph{(ii)} a predictive variance to quantify the prediction uncertainty.
With a GP surrogate model, we have the predictive mean $\mu(\x_* \vert \D) = K_{*\D} (K_{\D\D} + \sigma_n^2 I)^{-1} \mathbf{y}$ and variance $\sigma^2(\x_* \vert \D) = K_{**} - K_{*\D} (K_{\D\D} + \sigma_n^2 I)^{-1} K_{\D*}$
%
%
where $K_{**}=K(\x_*, \x_*)$, $[K_{*\D}]_{1,i} = K(\x_*, \x_i)$, $K_{\D*} = (K_{*\D})^T$, $[K_{\mathcal{DD}}]_{i,j} = K(\x_i, \x_j)$ and $\sigma_n^2$ is the noise variance.



\subsection{Combinatorial graphs and kernels}
\label{sect:bo_discrete}
\vspace{-4pt}

In BO on continuous search spaces the most popular surrogate models rely on GPs \citep{kandasamy2016gaussian, oh2018bock, snoek2012practical}.
Their popularity does not extend to combinatorial spaces, although kernels on combinatorial structures have also been proposed~\citep{haussler1999convolution, kondor2002diffusion, smola2003kernels}.
%
%
To design an effective GP-based \BO algorithm on combinatorial structures, a space of smooth functions --defined by the GP-- is needed.
%
%
We circumvent this requirement by the notion of the combinatorial graph defined as a graph, which contains all possible combinatorial choices as its vertices for a given combinatorial problem.
%
%
That is, each vertex corresponds to a different joint assignment of categorical or ordinal variables.
If two vertices are connected by an edge, then their respective set of combinatorial choices differ only by a single combinatorial choice.
%
%
As a consequence, we can now revisit the notion of smoothness on combinatorial structures as smoothness of a graph signal~\citep{chung1996spectral, ortega2018graph} defined on the combinatorial graph.
On a combinatorial graph, the shortest path is closely related to the Hamming distance.


%
%

\paragraph{The combinatorial graph}
%
To construct the combinatorial graph, we first define one sub-graph per combinatorial variable $C_i$, $\G(C_i)$.
For a categorical variable $C_i$, the sub-graph $\G(C_i)$ is chosen to be a complete graph while for an ordinal variable we have a path graph.
We aim at building a search space for combinatorial choices, \ie, a combinatorial graph, by combining sub-graphs $G(C_i)$ in such way that a distance between two adjacent vertices corresponds to a change of a value of a single combinatorial variable.
It turns out that the graph Cartesian product~\citep{hammack2011handbook} ensures this property.
Then, the graph Cartesian product of subgraphs $\G(C_j) = (\V_j, \E_j)$ is defined as $\G = (\V, \E) = \square_i \G(C_i)$, where $\V = \times_i \V_i$ and $(v_1=(c_1^{(1)}, \cdots, c_N^{(1)}),  v_2=(c_1^{(2)}, \cdots, c_N^{(2)})) \in \E$ if and only if $\exists j$ such that $\forall i\neq j$~$c_i^{(1)}=c_i^{(2)}$ and $(c_j^{(1)}, c_j^{(2)}) \in \E_j$.

As an example, let us consider a simplistic hyperparameter optimization problem for learning a neural network with three combinatorial variables: \emph{(i)} the batch size, $c_1 \in C_1 = \{16, 32, 64\}$, \emph{(ii)}~the optimizer $c_2 \in C_2 = \{AdaDelta, RMSProp, Adam\}$ and \emph{(iii)} the learning rate annealing $c_3 \in C_3 = \{Constant, Annealing\}$.
The sub-graphs $\{\G(C_i)\}_{i=1,2,3}$ for each of the combinatorial variables, as well as the final combinatorial graph after the graph Cartesian product, are illustrated in Figure~\ref{fig:combinatorial_graph}.
For the ordinal batch size variable we have a path graph, whereas for the categorical optimizer and learning rate annealing variables we have complete graphs.
The final combinatorial graph contains all possible combinations for batch size, optimizer and learning rate annealing.

\begin{figure}[!h]
    \centering
    \vskip -0.35cm
    \includegraphics[width=0.9\textwidth]{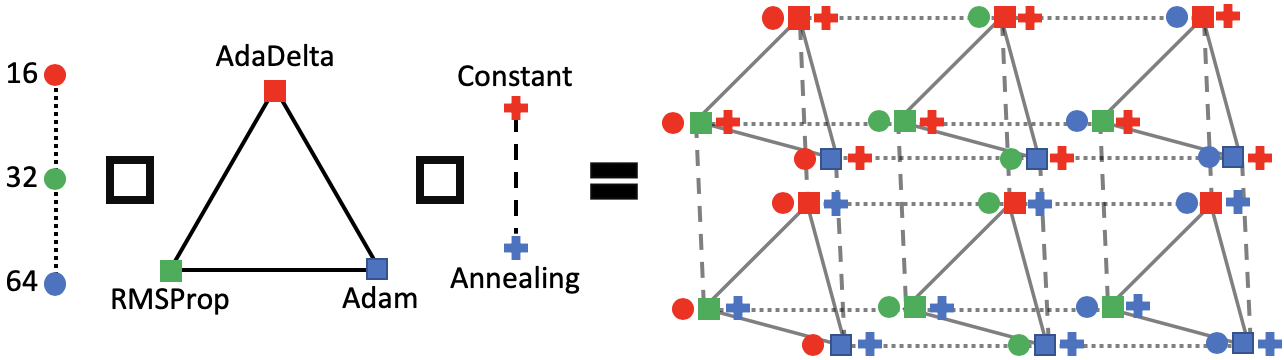}
    \vskip -0.3cm
    \caption{Combinatorial Graph: graph Cartesian product of sub-graphs $\G(C_1) \square \G(C_2) \square \G(C_3)$}
    \label{fig:combinatorial_graph}
    \vskip -0.4cm
\end{figure}



\paragraph{Cartesian product and Hamming distance}
The Hamming distance is a natural choice of distance on categorical variables.
With all complete sub-graphs, the shortest path between two vertices in the combinatorial graph is exactly equivalent to the Hamming distance between the respective categorical choices.
\begin{theorem}\label{thm:shortest_path_hamming}
Assume a combinatorial graph $\mathcal{G}=(\mathcal{V}, \mathcal{E})$ constructed from categorical variables, $C_1, \ldots, C_N$, that is, $\G$ is a graph Cartesian product $\square_i\G(C_i)$ of complete sub-graphs $\{\G(C_i)\}_i$.
Then the shortest path $s(v_1, v_2;\mathcal{G})$ between vertices $v_1=(c_1^{(1)}, \cdots, c_N^{(1)}),  v_2=(c_1^{(2)}, \cdots, c_N^{(2)}) \in \V$ on $\mathcal{G}$ is equal to the Hamming distance between $(c_1^{(1)}, \cdots, c_N^{(1)})$ and $(c_1^{(2)}, \cdots, c_N^{(2)})$.
\end{theorem}
\vspace{-8pt}
\begin{proof}
The proof of Theorem \ref{thm:shortest_path_hamming} could be found in Supp.~\ref{supp:sec_cartesian}
\end{proof}
\vspace{-4pt}
When there is a sub-graph which is not complete, the below result follows from the Thm. \ref{thm:shortest_path_hamming}:
\begin{corollary}
If a sub-graph is not a complete graph, then the shortest path is equal to or bigger than the Hamming distance. 
\end{corollary}
\vspace{-4pt}
The combinatorial graph using the graph Cartesian product is a natural search space for combinatorial variables that can encode a widely used metric on combinatorial variables like Hamming distance.

\vspace{-4pt}
\paragraph{Kernels on combinatorial graphs.}
In order to define the GP surrogate model for a combinatorial problem, we need to specify a a proper kernel on a combinatorial graph $\G = (\V, \E)$.
The role of the surrogate model is to smoothly interpolate and extrapolate neighboring data.
To define a smooth function on a graph, \ie, a smooth graph signal $f: \V \mapsto \mathbb{R}$, we adopt Graph Fourier Transforms (GFT) from graph signal processing~\citep{ortega2018graph}.
Similar to Fourier analysis on Euclidean spaces, GFT can represent any graph signal as a linear combination of graph Fourier bases. Suppressing the high frequency modes of the eigendecomposition approximates the signal with a smooth function on the graph.
We adopt the diffusion kernel which penalizes basis-functions in accordance with the magnitude of the frequency~\citep{kondor2002diffusion, smola2003kernels}.

To compute the diffusion kernel on the combinatorial graph $\G$, we need the eigensystem of graph Laplacian $L(\G) = \mathbf{D}_{\G} - \mathbf{A}_{\G}$, where $\mathbf{A}_{\G}$ is the adjacency matrix and $\mathbf{D}_{\G}$ is the degree matrix of the graph $\G$.
The eigenvalues $\{\lambda_1, \lambda_2, \cdots, \lambda_{\vert \V \vert}\}$ and eigenvectors $\{u_1, u_2, \cdots, u_{\vert \V \vert}\}$ of the graph Laplacian $L(\G)$ are the graph Fourier frequencies and bases, respectively.
Eigenvectors paired with large eigenvalues correspond to high-frequency Fourier bases.
The diffusion kernel is defined as
\vspace{-4pt}
\begin{equation}
    k([p],[q]\vert \beta) = \sum\nolimits_{i=1}^n e^{-\beta\lambda_{i}} u_{i} ([p]) u_{i} ([q]),
\end{equation}
from which it is clear that higher frequencies, $\lambda_i \gg 1$, are penalized more.
In a matrix form, with $\Lambda = diag(\lambda_1, \cdots, \lambda_{\vert \V \vert})$ and $\mathbf{U} = [u_1, \cdots, u_{\vert \V \vert}]$, the kernel takes the following form:
\vspace{-4pt}
\begin{equation} \label{eq:diffusion_kernel_matrix_form}
    K(\V, \V) = \mathbf{U} \exp(-\beta\Lambda) \mathbf{U}^T,
\end{equation}
which is the Gram matrix on all vertices whose submatrix is the Gram matrix for a subset of vertices.

\subsection{Scalable combinatorial Bayesian optimization with the graph Cartesian product}
\label{subsec:cartesian_product}

The direct computation of the diffusion kernel is infeasible because it involves the eigendecomposition of the Laplacian $L(\G)$, an operation with cubic complexity with respect to the number of vertices $\vert \V \vert$.
As we rely on the graph Cartesian product $\square_i \G_i$ to construct our combinatorial graph, we can take advantage of its properties and dramatically increase the efficiency of the eigendecomposition of the Laplacian $L(\G)$.
Further, due to the construction of the combinatorial graph, we can propose a variant of the diffusion kernel: automatic relevance determination (ARD) diffusion kernel.
The ARD diffusion kernel has more flexibility in its modeling capacity.
Moreover, in combination with the sparsity-inducing Horseshoe prior~\citep{carvalho2009handling} the ARD diffusion kernel performs variable selection automatically that allows to scale to high dimensional problems.

\paragraph{Speeding up the eigendecomposition with graph Cartesian products.}

Direct computation of the eigensystem of the Laplacian $L(\G)$ naively is infeasible, even for problems of moderate size. 
For instance, for 15 binary variables, eigendecomposition complexity is $O(\vert \V \vert^3)=(2^{15})^3$. 

The graph Cartesian product allows to improve the scalability of the eigendecomposition.
The Laplacian of the Cartesian product of two sub-graphs $\mathcal{G}_1$ and $\mathcal{G}_2$, $\mathcal{G}_1 \square \mathcal{G}_2$, can be algebraically expressed using the Kronecker product $\otimes$ and the Kronecker sum $\oplus$ \citep{hammack2011handbook}:
\vspace{-2pt}
\begin{equation}\label{eq:laplacian_cartesian}
    L(\G_1 \square \G_2) = L(\G_1) \oplus L(\G_2) = L(\G_1) \otimes \mathbf{I}_{1} + \mathbf{I}_{2} \otimes L(\G_2) ,
\end{equation}
where $\mathbf{I}$ denotes the identity matrix.
Considering the eigensystems $\{(\lambda_i^{(1)}, u_i^{(1)})\}$ and $\{(\lambda_j^{(2)}, u_j^{(2)})\}$ of $\mathcal{G}_1$ and $\mathcal{G}_2$, respectively, the eigensystem of $\mathcal{G}_1 \square \mathcal{G}_2$ is $\{(\lambda_i^{(1)} + \lambda_j^{(2)}, u_i^{(1)} \otimes u_j^{(2)})\}$.
Given Eq. \eqref{eq:laplacian_cartesian} and matrix exponentiation, for the diffusion kernel of $m$ categorical (or ordinal) variables we have 
\vspace{-2pt}
\begin{equation}\label{eq:cartesian_diffusion_kernel}
    \mathbf{K} = \exp\big(-\beta\; \bigoplus\nolimits_{i=1}^{m} L(\G_i)\big) = \bigotimes\nolimits_{i=1}^{m} \exp\big(-\beta\; L(\G_i)\big) .
\end{equation}
This means we can compute the kernel matrix by calculating the Kronecker product per sub-graph kernel.
Specifically, we obtain the kernel for the $i$-th sub-graph from the eigendecomposition of its Laplacian as per eq.~\eqref{eq:diffusion_kernel_matrix_form}.

Importantly, the decomposition of the final kernel into the Kronecker product of individual kernels in Eq.~\eqref{eq:cartesian_diffusion_kernel} leads to the following proposition.

\begin{proposition} \label{prop:efficiency}
Assume a graph $\mathcal{G}=(\mathcal{V}, \mathcal{E})$ is the graph cartesian product of sub-graphs $\G=\square_i, \G_i$. The graph Fourier Transform of $\G$ can be computed in $O(\sum_{i=1}^{m}|\mathcal{V}_i|^3)$ while the direct computation takes $O(\prod_{i=1}^{m}|\mathcal{V}_i|^3$).
\end{proposition}
\vspace{-6pt}
\noindent\emph{Proof.} The proof of Proposition \ref{prop:efficiency} could be found in the Supp.~\ref{supp:sec_cartesian}.

\paragraph{Variable-wise edge scaling.}
We can make the kernel more flexible by considering individual scaling factors $\{\beta_i\}$, a single $\beta_i$ for each variable.
The diffusion kernel then becomes:
\vspace{-4pt}
\begin{equation}\label{eq:ard_cartesian_diffusion_kernel}
    \mathbf{K} = \exp\big{(}\bigoplus\nolimits_{i=1}^{m} -\beta_i\;  L(\G_i)\big{)} = \bigotimes\nolimits_{i=1}^{m} \exp\big(-\beta_i\; L(\G_i)\big) ,
\end{equation}
where $\beta_i \ge 0$ for $i=1, \ldots, m$.
Since the diffusion kernel is a discrete version of the exponential kernel, the application of the individual $\beta_i$ for each variable is equivalent to the ARD kernel \citep{mackay1994bayesian, neal1995bayesian}.
Hence, we can perform variable (sub-graph) selection automatically.
We refer to this kernel as the \textit{ARD diffusion kernel}.

\paragraph{Prior on $\beta_i$.}
%
To determine $\beta_i$, and to prevent GP with ARD kernel from overfitting, we apply posterior sampling with a Horseshoe prior~\citep{carvalho2009handling} on the $\{\beta_i\}$.
The Horseshoe prior encourages sparsity, and, thus, enables variable selection, which, in turn, makes \combo statistically scalable to high dimensional problems.
For instance, if $\beta_i$ is set to zero, then $L(\G_i)$ does not contribute in Eq~\eqref{eq:ard_cartesian_diffusion_kernel}.
%
%

\subsection{\combo~algorithm}


We present the \combo approach in Algorithm~\ref{alg:combo}. More details about \combo could be found in the Supp. Sections \ref{supp:sec_surrogate} and \ref{supp:sec_acquisition}.

\setlength{\textfloatsep}{4pt}
\begin{algorithm}[tb]
  \caption{$\mathrm{COMBO}$: Combinatorial Bayesian Optimization on the combinatorial graph}
  \label{alg:combo}
\begin{algorithmic}[1]
  \STATE {\bfseries Input: $N$ combinatorial variables $\{C_i\}_{i=1,\cdots,N}$ } 
  \STATE Set a search space and compute Fourier frequencies and bases: \# See Sect. \ref{sect:bo_discrete}\\
  \STATE $\rhd$ Set sub-graphs $\G(C_i)$ for each variables $C_i$.\\
  \vspace{-3pt}
  \STATE $\rhd$ Compute eigensystem $\{(\lambda_k^{(i)}, u_k^{(i)})\}_{i,k}$ for each sub-graph $\G(C_i)$\\
  \STATE $\rhd$ Construct the combinatorial graph $\G=(\V,\E)=\square_i \G(C_i)$ using graph Cartesian product.\\
  \STATE Initialize $\mathcal{D}$.
  \REPEAT
    \STATE Fit GP using ARD diffusion kernel to $\D$ with slice sampling : $\mu(v_* \vert \D), \sigma^2(v_* \vert \D)$
    \STATE Maximize acquisition function : $v_{next} = \argmax_{v_* \in \V} a(\mu(v_* \vert \D), \sigma^2(v_* \vert \D))$
    \STATE Evaluate $f(v_{next})$, append to $\D = \D \cup \{(v_{next}, f(v_{next}))\}$
  \UNTIL{stopping criterion}
\end{algorithmic}
\end{algorithm}

We start the algorithm with defining all sub-graphs. Then, we calculate GFT (line 4 of Alg.~\ref{alg:combo}), whose result is needed to compute the ARD diffusion kernel, which could be sped up due to the application of the graph Cartesian product. Next, we fit the surrogate model parameters using slice sampling~\citep{murray2010slice, neal2003slice} (line 8).
Sampling begins with 100 steps of the burn-in phase. 
With the updated $\D$ of evaluated data, 10 points are sampled without thinning.
More details on the surrogate model fitting are given in Supp.~\ref{supp:sec_surrogate}.

Last, we maximize the acquisition function to find the next point for evaluation (line 9).
For this purpose, we begin with evaluating 20,000 randomly selected vertices. 
Twenty vertices with highest acquisition values are used as initial points for acquisition function optimization.
We use the breadth-first local search (BFLS), where at a given vertex we compare acquisition values on adjacent vertices.
We then move to the vertex with the highest acquisition value and repeat until no acquisition value on adjacent vertices are higher than the acquisition value at the current vertex.
BFLS is a local search, however, the initial random search and multi-starts help to escape from local minima.
In experiments (Supp.~\ref{supp:subsec_nonlocal}) we found that BFLS performs on par or better than non-local search, while being more efficient.

In our framework we can use any acquisition function like $\mathrm{GP\mh UBC}$, the Expected Improvement ($\mathrm{EI}$) \citep{rasmussen2006gaussian}, the predictive entropy search \citep{hernandez2014predictive} or knowledge gradient \citep{wu2017bayesian}.
We opt for $\mathrm{EI}$ that generally works well in practice \citep{shahriari2015taking}.

\section{Related work}
While for continuous inputs, $\mathcal{X} \subseteq \mathbb{R}^{D}$, there exist efficient algorithms to cope with high-dimensional search spaces using Gaussian processes(GPs) \citep{oh2018bock} or neural networks \citep{SRSKSSPPA:15}, few Bayesian Optimization(BO) algorithms have been proposed for combinatorial search spaces~\citep{baptista2018bayesian, bergstra2013making, hutter2011sequential}. 

A basic BO approach to combinatorial inputs is to represent all combinatorial variables using one-hot encoding and treating all integer-valued variables as values on a real line.
Further, for the integer-valued variables an acquisition function considers the closest integer for the chosen real value. 
This approach is used in Spearmint \citep{snoek2012practical}.
However, applying this method naively may result in severe problems, namely, the acquisition function could repeatedly evaluate the same points due to rounding real values to an integer and the one-hot representation of categorical variables. 
As pointed out in \citep{garrido2018dealing}, this issue could be fixed by making the objective constant over regions of input variables for which the actual objective has to be evaluated. 
The method was presented on a synthetic problem with two integer-valued variables, and a problem with one categorical variable and one integer-valued variable. 
Unfortunately, it remains unclear whether this approach is suitable for high-dimensional problems. 
Additionally, the proposed transformation of the covariance function seems to be better suited for ordinal-valued variables rather than categorical variables, further restricting the utility of this approach. 
In contrast, we propose a method that can deal with high-dimensional combinatorial (categorical and/or ordinal) spaces. 

Another approach to combinatorial optimization was proposed in BOCS~\citep{baptista2018bayesian} where the sparse Bayesian linear regression was used instead of GPs. 
The acquisition function was optimized by a semi-definite programming or simulated annealing that allowed to speed up the procedure of picking new points for next evaluations.
However, BOCS has certain limitations which restrict its application mostly to problems with low order interactions between variables.
BOCS requires users to specify the highest order of interactions among categorical variables, which inevitably ignores interaction terms of orders higher than the user-specified order.
Moreover, due to its parametric nature, the surrogate model of BOCS has excessively large number of parameters even for moderately high order (\eg, up to the 4th or 5th order).
Nevertheless, this approach achieved state-of-the-art results on four high-dimensional binary optimization problems. 
Different from \citep{baptista2018bayesian}, we use a non-parametric regression, \ie, {GP}s and perform variable selection both of which give more statistical efficiency.

%
%

\section{Experiments}
We evaluate \combo on two binary variable benchmarks, one ordinal and one multi-categorical variable benchmarks, as well as in two realistic problems: weighted Maximum Satisfiability and Neural Architecture Search.
We convert all into minimization problems.
We compare $\mathrm{SMAC}$~\citep{hutter2011sequential}, $\mathrm{TPE}$~\citep{bergstra2013making}, Simulated Annealing ($\mathrm{SA}$)~\citep{spears1993simulated}, as well as with $\mathrm{BOCS}$ ($\mathrm{BOCS\mh SDP}$ and $\mathrm{BOCS\mh SA3}$)\footnote{We exclude BOCS from ordinal/multi-categorical experiments, because at the time of the paper submission the open source implementation provided by the authors did not support ordinal/multi-categorical variables. For the explanation on how to use BOCS for ordinal/multi-categorical variables, please refer to the supplementary material of~\cite{baptista2018bayesian}.}~\cite{baptista2018bayesian}.
All details regarding experiments, baselines and results are in the supplementary material.
The code is available at: \url{https://github.com/QUVA-Lab/COMBO}

\subsection{Bayesian optimization with binary variables \footnote{In~\cite{oh2019combo}, the workshop version of this paper, we found that the methods were compared on different sets of initial evaluations and different objectives coming from the random processes involved in the generation of objectives, which turned out to be disadvantageous to \combo. We reran these experiments making sure that all methods are evaluated on the same set of 25 pairs of an objective and a set of initial evaluations.}}

\begin{table}[!th]
\centering
\caption{Results on the binary benchmarks (Mean $\pm$ Std.Err. over $25$ runs)}
\vspace{-0.2cm}
\begin{sc}
\resizebox{0.9\columnwidth}{!}{%
\begin{tabular}{l|ccc|ccc}
\toprule
& \multicolumn{3}{|c|}{Contamination control} & \multicolumn{3}{|c}{Ising sparsification} \\
Method & $\lambda=0$ & $\lambda=10^{-4}$ & $\lambda=10^{-2}$ & $\lambda=0$ & $\lambda=10^{-4}$ & $\lambda=10^{-2}$ \\
\midrule
$\mathrm{SMAC}$        & 21.61$\pm$0.04          & 21.50$\pm$0.03          & 21.68$\pm$0.04          & 0.152$\pm$0.040          & 0.219$\pm$0.052          & 0.350$\pm$0.045          \\
$\mathrm{TPE}$         & 21.64$\pm$0.04          & 21.69$\pm$0.04          & 21.84$\pm$0.04          & 0.404$\pm$0.109          & 0.444$\pm$0.095          & 0.609$\pm$0.107          \\
$\mathrm{SA}$          & 21.47$\pm$0.04          & 21.49$\pm$0.04          & 21.61$\pm$0.04          & \textbf{0.095}$\pm$0.033 & 0.117$\pm$0.035          & 0.334$\pm$0.064          \\
$\mathrm{BOCS\mh SDP}$ & 21.37$\pm$0.03          & 21.38$\pm$0.03          & 21.52$\pm$0.03          & 0.105$\pm$0.031          & \textbf{0.059}$\pm$0.013 & \textbf{0.300}$\pm$0.039 \\
\midrule
$\mathrm{COMBO}$       & \textbf{21.28}$\pm$0.03 & \textbf{21.28}$\pm$0.03 & \textbf{21.44}$\pm$0.03  & 0.103$\pm$0.035          & 0.081$\pm$0.028          & 0.317$\pm$0.042          \\
\bottomrule
\end{tabular}
}
\end{sc}
\label{tab:binary}
\end{table}

\paragraph{Contamination control} The contamination control in food supply chain is a binary optimization problem with 21 binary variables ($\approx 2.10\times10^6$ configurations) \citep{hu2010contamination}, where one can intervene at each stage of the supply chain to quarantine uncontaminated food with a cost.
The goal is to minimize food contamination while minimizing the prevention cost.
We set the budget to $270$ evaluations including $20$ random initial points.
We report results in Table \ref{tab:binary} and figures in Supp.~\ref{supp:exp_contamination}.
\combo outperforms all competing methods. 
Although the optimizing variables are binary, there exist higher order interactions among the variables due to the sequential nature of the problem, showcasing the importance of the modelling flexibility of \combo.

\paragraph{Ising sparsification}
A probability mass function(p.m.f) $p(x)$ can be defined by an Ising model $I_p$.
In Ising sparsification, we approximate the p.m.f $p(z)$ of $I_p$ with a p.m.f $q(z)$ of $I_q$.
%
%
The objective is the KL-divergence between $p$ and $q$ with a $\lambda$-parameterized regularizer: $\mathcal{L}(x) = D_{KL}(p||q) + \lambda \|x\|_{1}$. 
We consider 24 binary variable Ising models on $4 \times 4$ spin grid ($\approx 1.68\times10^7$ configurations) with a budget of $170$ evaluations, including $20$ random initial points.
We report results in Table \ref{tab:binary} and figures in Supp.~\ref{supp:exp_ising}.
We observe that \combo is competitive, obtaining slightly worse results, probably because in Ising sparsification there exist no complex interactions between variables.


\subsection{Bayesian optimization with ordinal and multi-categorical variables}





\begin{wraptable}[8]{r}{0.45\columnwidth}
\vspace{-28pt}
\begin{center}
\caption[caption]{Non-binary benchmarks results \\ \centering{(Mean $\pm$ Std.Err. over $25$ runs).}}
\label{tab:nonbinary}
\begin{sc}
\vspace{0.1cm}
\resizebox{0.4\columnwidth}{!}{%
\begin{tabular}{lcc}
\toprule
Method & Branin & Pest Control\\
\midrule
$\mathrm{SMAC}$     & 0.6962$\pm$0.0705          & 14.2614$\pm$0.0753          \\
$\mathrm{TPE}$      & 0.7578$\pm$0.0844          & 14.9776$\pm$0.0446          \\
$\mathrm{SA}$       & 0.4659$\pm$0.0211          & 12.7154$\pm$0.0918          \\
\midrule
$\mathrm{COMBO}$    & \textbf{0.4113}$\pm$0.0012 & \textbf{12.0012}$\pm$0.0033 \\
\bottomrule
\end{tabular}
}
\end{sc}
\end{center}
\vspace{-3pt}
\hspace{7pt}
\begin{minipage}[adjusting]{0.4\textwidth}
\begin{tiny}
\begin{spacing}{0.0}
We exclude BOCS, as the open source implementation provided by the authors does not support ordinal/multi-categorical variables.
\end{spacing}
\end{tiny}
\end{minipage}
\end{wraptable}

\paragraph{Ordinal variables}
The Branin benchmark is an optimization problem of a non-linear function over a 2D search space \citep{jones1998efficient}.
We discretize the search space, namely, we consider a grid of points that leads to an optimization problem with ordinal variables.
We set the budget to $100$ evaluations
and report results in Table \ref{tab:nonbinary} and Figure~\ref{supp:fig_branin} in the Supp.
\combo converges to a better solution faster and with better stability.

\paragraph{Multi-categorical variables}
The Pest control is a modified version of the contamination control with more complex, higher-order variable interactions, as detailed in Supp.~\ref{supp:exp_pestcontrol}.
We consider 21 pest control stations, each having 5 choices ($\approx 4.77\times10^{14}$ combinatorial choices).
We set the budget to $320$ including $20$ random initial points.
Results are in Table \ref{tab:nonbinary} and Figure~\ref{supp:fig_pest} in the Supp.
\combo outperforms all methods and converges faster.

\subsection{Weighted maximum satisfiability}

The satisfiability (SAT) problem is an important combinatorial optimization problem, where one decides how to set variables of a Boolean formula to make the formula true.
Many other optimization problems can be reformulated as SAT/MaxSAT problems.
Although highly successful, specialized MaxSAT solvers~\citep{maxsat} exist, we use MaxSAT as a testbed for \BO evaluation.
We run tests on three benchmarks from the Maximum Satisfiability Competition 2018.\footnote{\href{url}{https://maxsat-evaluations.github.io/2018/benchmarks.html}}
The wMaxSAT weights are unit normalized.
All evaluations are negated to obtain a minimization problem. 
We set the budget to $270$ evaluations including 20 random initial points. 
We report results in Table~\ref{tbl:wmaxsat} and Figures in Supp.~\ref{supp:exp_wmaxsat}, and runtimes on wMaxSAT43 in the figure next to Table~\ref{tbl:wmaxsat}.on wMaxSAT28~(Figure~\ref{supp:fig_wmaxsat28_runtime} in the Supp.)\footnote{The all runtimes were measured on Intel(R) Xeon(R) CPU E5-2630 v3 @ 2.40GHz with python codes.}



\vspace{-4pt}
\begin{table}[h!]
    \caption{(\textit{left}) Negated wMaxSAT Minimum and (\textit{right)} Runtime VS. Minimum on wMaxSAT43.}
    \begin{minipage}{0.6\textwidth}
        \vspace{-4pt}
        \resizebox{0.9\textwidth}{!}{%
        \begin{tabular}{lccc}
        \toprule
        Method & wMaxSAT28 & wMaxSAT43 & wMaxSAT60\\
        \midrule
        $\mathrm{SMAC}$              & -20.05$\pm$0.67         & -57.42$\pm$1.76         & -148.60$\pm$1.01\\
        $\mathrm{TPE}$               & -25.20$\pm$0.88         & -52.39$\pm$1.99         & -137.21$\pm$2.83\\
        $\mathrm{SA}$                & -31.81$\pm$1.19         & -75.76$\pm$2.30         & -187.55$\pm$1.50\\
        $\mathrm{BOCS\textrm{-}SDP}$ & -29.49$\pm$0.53         & -51.13$\pm$1.69         & -153.67$\pm$2.01\\
        $\mathrm{BOCS\textrm{-}SA3}$ & -34.79$\pm$0.78         & -61.02$\pm$2.28$^a$     & N.A$^b$\\
        \midrule
        $\mathrm{COMBO}$             & \textbf{-37.80}$\pm$0.27& \textbf{-85.02}$\pm$2.14& \textbf{-195.65}$\pm$0.00\\
        \bottomrule
        \end{tabular}
        }\\
        \footnotesize{$^a$ 270 evaluations were not finished after 168 hours.\\}
        \footnotesize{$^b$ Not tried due to the computation time longer than wMaxSAT43.}
    \end{minipage}
    \hspace{2pt}
    \begin{minipage}{0.5\textwidth}
        \vspace{-4pt}
        \resizebox{0.8\textwidth}{!}{%
        \includegraphics[width=\columnwidth]{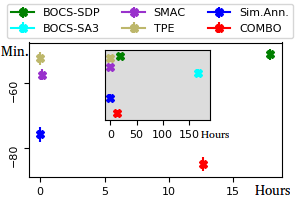}
        }
    \end{minipage}
    \label{tbl:wmaxsat}
    \vspace{-6pt}
\end{table}


\combo performs best in all cases. 
BOCS benefits from third-order interactions on wMaxSAT28 and wMaxSAT43.
However, this comes at the cost of large number of parameters~\citep{baptista2018bayesian}, incurring expensive computations.
When considering higher-order terms BOCS suffers severely from inefficient training.
This is due to a bad ratio between the number of parameters and number of training samples (\eg, for the 43 binary variables case BOCS-SA3/SA4/SA5 with, respectively, 3rd/4th/5th order interactions, has 13288/136698/1099296 parameters to train).
In contrast, \combo models arbitrarily high order interactions thanks to {\GP}'s nonparametric nature in a statistically efficient way.


Focusing on the largest problem, wMaxSAT60 with $\approx 1.15\times10^{18}$ configurations, \combo maintains superior performance.
We attribute this to the sparsity-inducing properties of the Horseshoe prior, after examining non sparsity-inducing priors (Supp.\ref{supp:exp_wmaxsat}).
The Horseshoe prior helps \combo attain further statistical efficiency.
We can interpret this reductionist behavior as the combinatorial version of methods exploiting low-effective dimensionality~\citep{bergstra2012random} on continuous search spaces~\citep{wang2016bayesian}.

The runtime --including evaluation time-- was measured on a dual 8-core 2.4 GHz (Intel Haswell E5-2630-v3) CPU with 64 GB memory using Python implementations.
SA, SMAC and TPE are faster but inaccurate compared to BOCS.
\combo is faster than BOCS-SA3, which needed 168 hours to collect around 200 evaluations.
\combo --modelling arbitrarily high-order interactions-- is also faster than BOCS-SDP constrained up-to second-order interactions only.

We conclude that in the realistic maximum satisfiablity problem \combo yields accurate solutions in reasonable runtimes, easily scaling up to high dimensional combinatorial search problems.

\vspace{-4pt}
\subsection{Neural architecture search}

\begin{wrapfigure}[14]{r}{0.45\columnwidth}
    \vspace{-8pt}
    \includegraphics[width=0.4\columnwidth]{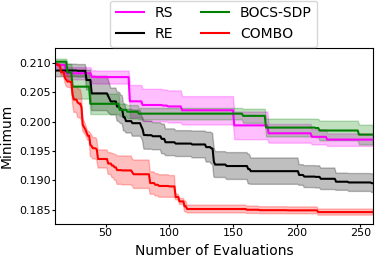}
    \vspace{-10pt}
    \caption[caption]{Result for Neural Architecture Search \centering(Mean $\pm$ Std.Err. over $4$ runs)}
    \label{fig:nasbinary}
\end{wrapfigure}

Last, we compare \BO methods on a neural architecture search (NAS) problem, a typical combinatorial optimization problem \citep{wistuba2019survey}.
We compare \combo with BOCS, as well as Regularized Evolution (RE)~\citep{real2018regularized}, one of the most successful evolutionary search algorithm for NAS~\citep{wistuba2019survey}.
We  include Random Search (RS) which can be competitive in well-designed search spaces~\citep{wistuba2019survey}.
We do not compare with the \BO-based NASBOT~\citep{kandasamy2018neural}.
NASBOT focuses exclusively on NAS problems and optimizes over a different search space than ours using an optimal transport-based metric between architectures, which is out of the scope for this work.

\begin{table}[!h]
    \caption{(\textit{left}) Connectivity (X -- no connection, O -- states are connected), (\textit{right)} Computation type.}
    \vspace{-4pt}
    \begin{minipage}{0.5\textwidth}
    \begin{center}
    \begin{scriptsize}
    \begin{sc}
    \begin{tabular}{c|ccccccc}
        & \textbf{IN} & \textbf{H1} & \textbf{H2} & \textbf{H3} & \textbf{H4} & \textbf{H5} & \textbf{OUT}\\
        \midrule
        \textbf{IN} & - & O & X & X & X & O & X \\
        \textbf{H1} & - & - & X & O & X & X & O \\
        \textbf{H2} & - & - & - & X & O & X & X \\
        \textbf{H3} & - & - & - & - & X & O & X \\
        \textbf{H4} & - & - & - & - & - & O & O \\
        \textbf{H5} & - & - & - & - & - & - & X \\
        \textbf{OUT}& - & - & - & - & - & - & - \\
    \end{tabular}  
    \end{sc}
    \end{scriptsize}
    \end{center}
    \end{minipage}
    \hspace{-2pt}
    \begin{minipage}{0.5\textwidth}
    \begin{footnotesize}
    \begin{sc}
    \begin{tabular}{c|c|c}
        \toprule
            & \textbf{MaxPool} & \textbf{Conv}\\
        \midrule
        \textbf{Small}  & Id $\equiv$ MaxPool(1$\times$1)& Conv(3$\times$3)\\
        \midrule
        \textbf{Large}  & MaxPool(3$\times$3)            & Conv(5$\times$5)\\
        \bottomrule
    \end{tabular}    
    \end{sc}
    \end{footnotesize}
    \end{minipage}
    \label{tbl:nas_search_space}
\end{table}

For the considered NAS problem we aim at finding the optimal cell comprising of one input node (\textbf{IN}), one output node (\textbf{OUT}) and five possible hidden nodes (\textbf{H1}--\textbf{H5}). 
We allow connections from \textbf{IN} to all other nodes, from \textbf{H1} to all other nodes and so one. 
We exclude connections that could cause loops.
An example of connections within a cell can be found in Table.~\ref{tbl:nas_search_space} on the left, where the input state \textbf{IN} connects to \textbf{H1}, \textbf{H1} connects to \textbf{H3} and \text{OUT}, and so on.
The input state and output state have identity computation types, whereas the computation types for the hidden states are determined by combination of 2 binary choices from the table on the right of Table.~\ref{tbl:nas_search_space}. 
In total, the search space consists of 31 binary variables, 21 for the connectivities and 2 for 5 computation types.
%

The objective is to minimize the classification error on validation set of CIFAR10~\citep{krizhevsky2009learning} with a penalty on the amount of FLOPs of a neural network constructed with a given cell.
We search for an architecture that balances accuracy and computational efficiency.
In each evaluation, we construct a cell, and stack three cells to build a final neural network.
More details are given in the Supp.~\ref{supp:exp_nas}.

In Figure~\ref{fig:nasbinary} we can notice that \combo outperforms other methods significantly.
BOCS-SDP and RS exhibit similar performance, confirming that for NAS modeling high-order interactions between variables is crucial.
Furthermore, \combo outperforms the specialized RE, one of the most successful evolutionary search (ES) algorithms shown to perform better on NAS than reinforcement learning (RL) algorithms~\citep{real2018regularized, wistuba2019survey}.
When increasing the number of evaluations to 500, RE still cannot reach the performance of \combo with 260 evaluations, see Figure~\ref{supp:fig_nas_500} in the Supp.
A possible explanation for such behavior is the high sensitivity to choices of hyperparameters of RE, and ES requires far more evaluations in general.
Details about RE hyperparameters can be found in the Supp.~\ref{supp:exp_nas}.

Due to the difficulty of using BO on combinatoral structures, BOs have not been widely used for NAS with few exceptions~\citep{kandasamy2018neural}.
\combo's performance suggests that a well-designed general combinatorial BO can be competitive or even better in NAS than ES and RL, especially when computational resources are constrained.
Since \combo is applicable to any set of combinatorial variables, its use in NAS is not restricted to the typical NASNet search space.
Interestingly, \combo can approximately optimize continuous variables by discretization, as shown in the ordinal variable experiment, thus, jointly optimizing the architecture and hyperparameter learning.


\section{Conclusion}
In this work, we propose COMBO, a Bayesian Optimization method for combinatorial search spaces.
To the best of our knowledge, COMBO is the first Bayesian Optimization algorithm using Gaussian Processes as a surrogate model suitable for problems with complex high order interactions between variables.
To efficiently tackle the exponentially increasing complexity of combinatorial search spaces, we rest upon the following ideas:
\emph{(i)} we represent the search space as the combinatorial graph, which combines sub-graphs given to all combinatorial variables using the graph Cartesian product. 
Moreover, the combinatorial graph reflects a natural metric on categorical choices (Hamming distance) when all combinatorial variables are categorical.
\emph{(ii)} we adopt the GFT to define the ``smoothness'' of functions on combinatorial structures.
\emph{(iii)} we propose a flexible ARD diffusion kernel for GPs on the combinatorial graph with a Horseshoe prior on scale parameters, which makes \combo scale up to high dimensional problems by performing variable selection.
All above features together make that \combo outperforms competitors consistently on a wide range of problems.
\combo is a statistically and computationally scalable Bayesian Optimization tool for combinatorial spaces, which is a field that has not been extensively explored.



\begin{thebibliography}{49}
\providecommand{\natexlab}[1]{#1}
\providecommand{\url}[1]{\texttt{#1}}
\expandafter\ifx\csname urlstyle\endcsname\relax
  \providecommand{\doi}[1]{doi: #1}\else
  \providecommand{\doi}{doi: \begingroup \urlstyle{rm}\Url}\fi

\bibitem[Bacchus et~al.(2018)Bacchus, Jarvisalo, and Martins]{maxsat}
F.~Bacchus, M.~Jarvisalo, and R.~Martins, editors.
\newblock \emph{{MaxSAT Evaluation 2018. Solver and Benchmark Descriptions}},
  volume B-2018-2, 2018. University of Helsinki.

\bibitem[Baptista and Poloczek(2018)]{baptista2018bayesian}
R.~Baptista and M.~Poloczek.
\newblock {Bayesian Optimization of Combinatorial Structures}.
\newblock In \emph{International Conference on Machine Learning}, pages
  462--471, 2018.

\bibitem[Bergstra and Bengio(2012)]{bergstra2012random}
J.~Bergstra and Y.~Bengio.
\newblock Random search for hyper-parameter optimization.
\newblock \emph{Journal of Machine Learning Research}, 13\penalty0
  (Feb):\penalty0 281--305, 2012.

\bibitem[Bergstra et~al.(2013)Bergstra, Yamins, and Cox]{bergstra2013making}
J.~Bergstra, D.~Yamins, and D.~D. Cox.
\newblock Making a science of model search: Hyperparameter optimization in
  hundreds of dimensions for vision architectures.
\newblock 2013.

\bibitem[Brochu et~al.(2010)Brochu, Cora, and De~Freitas]{brochu2010tutorial}
E.~Brochu, V.~M. Cora, and N.~De~Freitas.
\newblock A tutorial on bayesian optimization of expensive cost functions, with
  application to active user modeling and hierarchical reinforcement learning.
\newblock \emph{arXiv preprint arXiv:1012.2599}, 2010.

\bibitem[Carvalho et~al.(2009)Carvalho, Polson, and
  Scott]{carvalho2009handling}
C.~M. Carvalho, N.~G. Polson, and J.~G. Scott.
\newblock Handling sparsity via the horseshoe.
\newblock In \emph{Artificial Intelligence and Statistics}, pages 73--80, 2009.

\bibitem[Carvalho et~al.(2010)Carvalho, Polson, and
  Scott]{carvalho2010horseshoe}
C.~M. Carvalho, N.~G. Polson, and J.~G. Scott.
\newblock The horseshoe estimator for sparse signals.
\newblock \emph{Biometrika}, 97\penalty0 (2):\penalty0 465--480, 2010.

\bibitem[Chung(1996)]{chung1996spectral}
F.~R. Chung.
\newblock Spectral graph theory (cbms regional conference series in
  mathematics, no. 92).
\newblock 1996.

\bibitem[Davis(1991)]{davis1991handbook}
L.~Davis.
\newblock Handbook of genetic algorithms.
\newblock 1991.

\bibitem[Elsken et~al.(2018)Elsken, Metzen, and Hutter]{elsken2018neural}
T.~Elsken, J.~H. Metzen, and F.~Hutter.
\newblock Neural architecture search: A survey.
\newblock \emph{arXiv preprint arXiv:1808.05377}, 2018.

\bibitem[Frazier(2018)]{frazier2018tutorial}
P.~I. Frazier.
\newblock A tutorial on bayesian optimization.
\newblock \emph{arXiv preprint arXiv:1807.02811}, 2018.

\bibitem[Freitas(2009)]{freitas2009review}
A.~A. Freitas.
\newblock A review of evolutionary algorithms for data mining.
\newblock In \emph{Data Mining and Knowledge Discovery Handbook}, pages
  371--400. Springer, 2009.

\bibitem[Garnett et~al.(2010)Garnett, Osborne, and
  Roberts]{garnett2010bayesian}
R.~Garnett, M.~A. Osborne, and S.~J. Roberts.
\newblock Bayesian optimization for sensor set selection.
\newblock In \emph{Proceedings of the 9th ACM/IEEE international conference on
  information processing in sensor networks}, pages 209--219. ACM, 2010.

\bibitem[Garrido-Merch{\'a}n and
  Hern{\'a}ndez-Lobato(2018)]{garrido2018dealing}
E.~C. Garrido-Merch{\'a}n and D.~Hern{\'a}ndez-Lobato.
\newblock Dealing with categorical and integer-valued variables in bayesian
  optimization with gaussian processes.
\newblock \emph{arXiv preprint arXiv:1805.03463}, 2018.

\bibitem[Hammack et~al.(2011)Hammack, Imrich, and
  Klav{\v{z}}ar]{hammack2011handbook}
R.~Hammack, W.~Imrich, and S.~Klav{\v{z}}ar.
\newblock \emph{Handbook of product graphs}.
\newblock CRC press, 2011.

\bibitem[Hansen and Jaumard(1990)]{hansen1990algorithms}
P.~Hansen and B.~Jaumard.
\newblock Algorithms for the maximum satisfiability problem.
\newblock \emph{Computing}, 44\penalty0 (4):\penalty0 279--303, 1990.

\bibitem[Haussler(1999)]{haussler1999convolution}
D.~Haussler.
\newblock Convolution kernels on discrete structures.
\newblock Technical report, Technical report, Department of Computer Science,
  University of California, 1999.

\bibitem[Hern{\'a}ndez-Lobato et~al.(2014)Hern{\'a}ndez-Lobato, Hoffman, and
  Ghahramani]{hernandez2014predictive}
J.~M. Hern{\'a}ndez-Lobato, M.~W. Hoffman, and Z.~Ghahramani.
\newblock Predictive entropy search for efficient global optimization of
  black-box functions.
\newblock In \emph{Advances in neural information processing systems}, pages
  918--926, 2014.

\bibitem[Hu et~al.(2010)Hu, Hu, Xu, Wang, and Cao]{hu2010contamination}
Y.~Hu, J.~Hu, Y.~Xu, F.~Wang, and R.~Z. Cao.
\newblock Contamination control in food supply chain.
\newblock In \emph{Simulation Conference (WSC), Proceedings of the 2010
  Winter}, pages 2678--2681. IEEE, 2010.

\bibitem[Hutter et~al.(2011)Hutter, Hoos, and
  Leyton-Brown]{hutter2011sequential}
F.~Hutter, H.~H. Hoos, and K.~Leyton-Brown.
\newblock Sequential model-based optimization for general algorithm
  configuration.
\newblock In \emph{International Conference on Learning and Intelligent
  Optimization}, pages 507--523. Springer, 2011.

\bibitem[Jones et~al.(1998)Jones, Schonlau, and Welch]{jones1998efficient}
D.~R. Jones, M.~Schonlau, and W.~J. Welch.
\newblock Efficient global optimization of expensive black-box functions.
\newblock \emph{Journal of Global optimization}, 13\penalty0 (4):\penalty0
  455--492, 1998.

\bibitem[Kandasamy et~al.(2016)Kandasamy, Dasarathy, Oliva, Schneider, and
  P{\'o}czos]{kandasamy2016gaussian}
K.~Kandasamy, G.~Dasarathy, J.~B. Oliva, J.~Schneider, and B.~P{\'o}czos.
\newblock Gaussian process bandit optimisation with multi-fidelity evaluations.
\newblock In \emph{Advances in Neural Information Processing Systems}, pages
  992--1000, 2016.

\bibitem[Kandasamy et~al.(2018)Kandasamy, Neiswanger, Schneider, Poczos, and
  Xing]{kandasamy2018neural}
K.~Kandasamy, W.~Neiswanger, J.~Schneider, B.~Poczos, and E.~P. Xing.
\newblock Neural architecture search with bayesian optimisation and optimal
  transport.
\newblock In \emph{Advances in Neural Information Processing Systems}, pages
  2016--2025, 2018.

\bibitem[Kingma and Ba(2014)]{kingma2014adam}
D.~P. Kingma and J.~Ba.
\newblock Adam: A method for stochastic optimization.
\newblock \emph{arXiv preprint arXiv:1412.6980}, 2014.

\bibitem[Kondor and Lafferty(2002)]{kondor2002diffusion}
R.~I. Kondor and J.~Lafferty.
\newblock Diffusion kernels on graphs and other discrete structures.
\newblock In \emph{Proceedings of the 19th international conference on machine
  learning}, volume 2002, pages 315--322, 2002.

\bibitem[Krizhevsky and Hinton(2009)]{krizhevsky2009learning}
A.~Krizhevsky and G.~Hinton.
\newblock Learning multiple layers of features from tiny images.
\newblock Technical report, Citeseer, 2009.

\bibitem[MacKay(1994)]{mackay1994bayesian}
D.~J. MacKay.
\newblock Bayesian nonlinear modeling for the prediction competition.
\newblock \emph{ASHRAE transactions}, 100\penalty0 (2):\penalty0 1053--1062,
  1994.

\bibitem[Malkomes et~al.(2016)Malkomes, Schaff, and
  Garnett]{malkomes2016bayesian}
G.~Malkomes, C.~Schaff, and R.~Garnett.
\newblock Bayesian optimization for automated model selection.
\newblock In \emph{Advances in Neural Information Processing Systems}, pages
  2900--2908, 2016.

\bibitem[Mo{\v{c}}kus(1975)]{movckus1975bayesian}
J.~Mo{\v{c}}kus.
\newblock On bayesian methods for seeking the extremum.
\newblock In \emph{Optimization Techniques IFIP Technical Conference}, pages
  400--404. Springer, 1975.

\bibitem[Murray and Adams(2010)]{murray2010slice}
I.~Murray and R.~P. Adams.
\newblock Slice sampling covariance hyperparameters of latent gaussian models.
\newblock In \emph{Advances in neural information processing systems}, pages
  1732--1740, 2010.

\bibitem[Neal(1995)]{neal1995bayesian}
R.~M. Neal.
\newblock \emph{Bayesian learning for neural networks}.
\newblock PhD thesis, University of Toronto, 1995.

\bibitem[Neal(2003)]{neal2003slice}
R.~M. Neal.
\newblock Slice sampling.
\newblock \emph{Annals of Statistics}, pages 705--741, 2003.

\bibitem[Oh et~al.(2018)Oh, Gavves, and Welling]{oh2018bock}
C.~Oh, E.~Gavves, and M.~Welling.
\newblock Bock: Bayesian optimization with cylindrical kernels.
\newblock \emph{arXiv preprint arXiv:1806.01619}, 2018.

\bibitem[Oh et~al.(2019)Oh, Tomczak, Gavves, and Welling]{oh2019combo}
C.~Oh, J.~Tomczak, E.~Gavves, and M.~Welling.
\newblock Combo: Combinatorial bayesian optimization using graph
  representations.
\newblock In \emph{ICML Workshop on Learning and Reasoning with
  Graph-Structured Data}, 2019.

\bibitem[Ortega et~al.(2018)Ortega, Frossard, Kova{\v{c}}evi{\'c}, Moura, and
  Vandergheynst]{ortega2018graph}
A.~Ortega, P.~Frossard, J.~Kova{\v{c}}evi{\'c}, J.~M. Moura, and
  P.~Vandergheynst.
\newblock Graph signal processing: Overview, challenges, and applications.
\newblock \emph{Proceedings of the IEEE}, 106\penalty0 (5):\penalty0 808--828,
  2018.

\bibitem[Paszke et~al.(2017)Paszke, Gross, Chintala, Chanan, Yang, DeVito, Lin,
  Desmaison, Antiga, and Lerer]{paszke2017automatic}
A.~Paszke, S.~Gross, S.~Chintala, G.~Chanan, E.~Yang, Z.~DeVito, Z.~Lin,
  A.~Desmaison, L.~Antiga, and A.~Lerer.
\newblock Automatic differentiation in pytorch.
\newblock 2017.

\bibitem[Rasmussen and Williams(2006)]{rasmussen2006gaussian}
C.~E. Rasmussen and C.~K. Williams.
\newblock \emph{Gaussian processes for machine learning}.
\newblock the MIT Press, 2006.

\bibitem[Real et~al.(2018)Real, Aggarwal, Huang, and Le]{real2018regularized}
E.~Real, A.~Aggarwal, Y.~Huang, and Q.~V. Le.
\newblock Regularized evolution for image classifier architecture search.
\newblock \emph{arXiv preprint arXiv:1802.01548}, 2018.

\bibitem[Resende et~al.(1997)Resende, Pitsoulis, and
  Pardalos]{resende1997approximate}
M.~G. Resende, L.~Pitsoulis, and P.~Pardalos.
\newblock Approximate solution of weighted max-sat problems using grasp.
\newblock \emph{Satisfiability problems}, 35:\penalty0 393--405, 1997.

\bibitem[Shahriari et~al.(2015)Shahriari, Swersky, Wang, Adams, and
  De~Freitas]{shahriari2015taking}
B.~Shahriari, K.~Swersky, Z.~Wang, R.~P. Adams, and N.~De~Freitas.
\newblock Taking the human out of the loop: A review of bayesian optimization.
\newblock \emph{Proceedings of the IEEE}, 104\penalty0 (1):\penalty0 148--175,
  2015.

\bibitem[Smola and Kondor(2003)]{smola2003kernels}
A.~J. Smola and R.~Kondor.
\newblock Kernels and regularization on graphs.
\newblock In \emph{Learning theory and kernel machines}, pages 144--158.
  Springer, 2003.

\bibitem[Snoek et~al.(2012)Snoek, Larochelle, and Adams]{snoek2012practical}
J.~Snoek, H.~Larochelle, and R.~P. Adams.
\newblock Practical bayesian optimization of machine learning algorithms.
\newblock In \emph{Advances in neural information processing systems}, pages
  2951--2959, 2012.

\bibitem[Snoek et~al.(2014)Snoek, Swersky, Zemel, and Adams]{snoek2014input}
J.~Snoek, K.~Swersky, R.~Zemel, and R.~Adams.
\newblock Input warping for bayesian optimization of non-stationary functions.
\newblock In \emph{International Conference on Machine Learning}, pages
  1674--1682, 2014.

\bibitem[Snoek et~al.(2015)Snoek, Rippel, Swersky, Kiros, Satish, Sundaram,
  Patwary, Prabhat, and Adams]{SRSKSSPPA:15}
J.~Snoek, O.~Rippel, K.~Swersky, R.~Kiros, N.~Satish, N.~Sundaram, M.~Patwary,
  M.~Prabhat, and R.~Adams.
\newblock Scalable bayesian optimization using deep neural networks.
\newblock In \emph{International Conference on Machine Learning}, pages
  2171--2180, 2015.

\bibitem[Spears(1993)]{spears1993simulated}
W.~M. Spears.
\newblock Simulated annealing for hard satisfiability problems.
\newblock In \emph{Cliques, Coloring, and Satisfiability}, pages 533--558.
  Citeseer, 1993.

\bibitem[Wang et~al.(2016)Wang, Hutter, Zoghi, Matheson, and
  de~Feitas]{wang2016bayesian}
Z.~Wang, F.~Hutter, M.~Zoghi, D.~Matheson, and N.~de~Feitas.
\newblock Bayesian optimization in a billion dimensions via random embeddings.
\newblock \emph{Journal of Artificial Intelligence Research}, 55:\penalty0
  361--387, 2016.

\bibitem[Wilson et~al.(2014)Wilson, Fern, and Tadepalli]{wilson2014using}
A.~Wilson, A.~Fern, and P.~Tadepalli.
\newblock {Using trajectory data to improve Bayesian optimization for
  reinforcement learning}.
\newblock \emph{The Journal of Machine Learning Research}, 15\penalty0
  (1):\penalty0 253--282, 2014.

\bibitem[Wistuba et~al.(2019)Wistuba, Rawat, and Pedapati]{wistuba2019survey}
M.~Wistuba, A.~Rawat, and T.~Pedapati.
\newblock A survey on neural architecture search.
\newblock \emph{arXiv preprint arXiv:1905.01392}, 2019.

\bibitem[Wu et~al.(2017)Wu, Poloczek, Wilson, and Frazier]{wu2017bayesian}
J.~Wu, M.~Poloczek, A.~G. Wilson, and P.~Frazier.
\newblock Bayesian optimization with gradients.
\newblock In \emph{Advances in Neural Information Processing Systems}, pages
  5267--5278, 2017.

\end{thebibliography}

\newpage

\setcounter{page}{1}
\setcounter{section}{0}


\vbox{%
    \hsize\textwidth
    \linewidth\hsize
    \vskip 0.1in
    \hrule height 4pt
    \vskip 0.25in
    \vskip -\parskip%
    \centering
    {\LARGE\bf Combinatorial Bayesian Optimization\\using the Graph Cartesian Product\\Supplementary Material\par}
    \vskip 0.29in
    \vskip -\parskip
    \hrule height 1pt
    \vskip 0.09in
    
  }

\section{Graph Cartesian product}
\label{supp:sec_cartesian}
\subsection{Graph Cartesian product and Hamming distance}

\begin{theorem}\label{supp:thm_shortest_path_hamming}
Assume a combinatorial graph $\mathcal{G}=(\mathcal{V}, \mathcal{E})$ constructed from categorical variables, $C_1, \cdots, C_N$, that is, $\G$ is a graph Cartesian product $\square_i\G(C_i)$ of complete sub-graphs $\{\G(C_i)\}_i$.
Then the shortest path $s(v_1, v_2;\mathcal{G})$ between vertices $v_1=(c_1^{(1)}, \cdots, c_N^{(1)}),  v_2=(c_1^{(2)}, \cdots, c_N^{(2)}) \in \V$ on $\mathcal{G}$ is equal to the Hamming distance between $(c_1^{(1)}, \cdots, c_N^{(1)})$ and $(c_1^{(2)}, \cdots, c_N^{(2)})$.
\end{theorem}
\begin{proof}
From the graph Cartesian product definition we have that the shortest path between $v_1$ and $v_2$ consists of edges that change a value in one categorical variable at a time. As a result, an edge between $c_i^{(1)}$ and $c_i^{(2)}$, \ie, a difference in the i-$th$ categorical variable, and all other edges fixed contributes one error to the Hamming distance. Therefore, we can define the shortest path between $v_1$ and $v_2$ as the sum over all edges for which $c_i^{(1)}$ and $c_i^{(2)}$ are different, $s(v_1, v_2;\mathcal{G}) = \sum_i \mathds{1}[c_i^{(1)} \neq c_i^{(2)}]$ that is equivalent to the definition of the Hamming distance between two sets of categorical choices.
\end{proof}

\subsection{Graph Fourier transform with graph Cartesian product}
Graph Cartesian products can help us improve the scalability of the eigendecomposition~\citep{hammack2011handbook}. 
The Laplacian of the Cartesian product $\mathcal{G}_1 \square \mathcal{G}_2$ of two sub-graphs $\mathcal{G}_1$ and $\mathcal{G}_2$ can be algebraically expressed using the Kronecker product $\otimes$ and the Kronecker sum $\oplus$ \citep{hammack2011handbook}:
\vspace{-2pt}
\begin{equation}\label{supp_eq_laplacian_cartesian}
    L(\G_1 \square \G_2) = L(\G_1) \oplus L(\G_2) = L(\G_1) \otimes \mathbf{I}_{1} + \mathbf{I}_{2} \otimes L(\G_2) ,
\end{equation}
where $\mathbf{I}$ denotes the identity matrix.
As a consequence, considering the eigensystems $\{(\lambda_i^{(1)}, u_i^{(1)})\}$ and $\{(\lambda_j^{(2)}, u_j^{(2)})\}$ of $\mathcal{G}_1$ and $\mathcal{G}_2$, respectively, the eigensystem of $\mathcal{G}_1 \square \mathcal{G}_2$ is $\{(\lambda_i^{(1)} + \lambda_j^{(2)}, u_i^{(1)} \otimes u_j^{(2)})\}$.

\begin{proposition} \label{supp:prop_efficiency}
Assume a graph $\mathcal{G}=(\mathcal{V}, \mathcal{E})$ is the graph cartesian product of sub-graphs $\G=\square_i, \G_i$. Then graph Fourier Transform of $\G$ can be computed in $O(\sum_{i=1}^{m}|\mathcal{V}_i|^3)$ while the direct computation takes $O(\prod_{i=1}^{m}|\mathcal{V}_i|^3$).
\end{proposition}
\begin{proof}
Graph Fourier Transform is eigendecomposition of graph Laplacian $L(\G)$ where $\G=(\V, \E)$.
Eigendecomposition is of cubic complexity with respect to the number of rows(= the number of columns), which is the number of vertices $\vert \V \vert$ for graph Laplacian $L(\G)$.
If we directly compute eigendecomposition of $L(\G)$, it costs $O(\prod_i \vert \V \vert^3)$.
If we utilize graph Cartesian product, then we compute eigendecomposition for each sub-graphs and combine those to obtain eigendecomposition of the original full graph $\G$.
The cost for eigendecomposition of each subgraphs is $O(\vert \V_i \vert^3)$ and in total, it is summed to $O(\sum_i \vert \V \vert^3)$.
For graph Cartesian product, graph Fourier Transform can be computed in $O(\sum_i \vert \V \vert^3)$.
\end{proof}
\vspace{-12pt}
\paragraph{\textbf{Remark}} In the computation of gram matrices, eigenvalues from sub-graphs are summed and entries of eigenvectors are multiplied.
Compared to the cost of $O(\prod_i \vert \V \vert^3)$, this overhead is marginal.
Thus with graph Cartesian product, the ARD diffusion kernel can be computed efficiently with a pre-computed eigensystem for each sub-graphs.
This pre-computation is performed efficiently by using Prop.~\ref{supp:prop_efficiency}

\section{Surrogate model fitting}
\label{supp:sec_surrogate}
In the surrogate model fitting step of \combo, GP-parameters are sampled from the posterior using slice sampling~\citep{murray2010slice, neal2003slice} as in Spearmint~\citep{snoek2012practical, snoek2014input}.

\subsection{GP-parameter posterior sampling}

For a nonzero mean function, the marginal likelihood of $\D=(V,\mathbf{y})$ is
\begin{equation}
    -\frac{1}{2}(\mathbf{y}-m)^T (\sigma_f^2 K_{VV} + \sigma_n^2 I)^{-1} (\mathbf{y}-m) - \frac{1}{2}\log \det(\sigma_f^2 K_{VV} + \sigma_n^2 I) - \frac{n}{2} \log 2\pi 
\end{equation}
where $m$ is the value of constant mean function.
With \ARD diffusion kernel, the gram matrix is given by
\begin{equation}\label{eq:diffusion_kernel_supp}
    \sigma_f^2 K_{VV} + \sigma_n^2 I = \sigma_f^2 \bigotimes_i U_i \exp^{-\beta_i \Lambda_i} U_i^T + \sigma_n^2 I
\end{equation}
where $\Lambda_i$ is a diagonal matrix whose diagonal entries are eigenvalues of a sub-graph given to a combinatorial variable $L(\G(C_i))$, $U_i$ is a orthogonal matrix whose columns are corresponding eigenvalues of $L(\G(C_i))$, signal variance $\sigma_f^2$ and noise variance $\sigma_n^2$.

\paragraph{\textbf{Remark}} In the implementation of eq.~(\ref{eq:diffusion_kernel_supp}), a normalized version $\exp^{-\beta_i \Lambda_i} / \Psi_i$ where $\Psi_i = {1/\vert \mathcal{V}_i \vert \sum_{j=1, \cdots \vert \mathcal{V}_i \vert} \exp^{-\beta_i \lambda^{(i)}_j}}$ is used for numerical stability instead of $\exp^{-\beta_i \Lambda_i}$.

In the surrogate model fitting step of \combo, all GP-parameters are sampled from posterior which is proportional to the product of above marginal likelihood and priors on all GP-parameters such as $\beta_i$'s, signal variance $\sigma_f^2$, noise variance $\sigma_n^2$ and constant mean function value $m$.
In \combo all GP-parameters are sampled using slice sampling~\citep{neal2003slice}.

A single step of slice sampling in \combo consists of multiple univariate slice sampling steps:
\begin{itemize}[itemsep=-0.2ex, topsep=-3pt, partopsep=-2pt]
    \item [1.] $m$ : constant mean function value $m$
    \item [2.] $\sigma_f^2$ : signal variance
    \item [3.] $\sigma_n^2$ : noise variance
    \item [4.] $\{\beta_i\}_i$ with a randomly shuffled order
\end{itemize}
In \combo, slice sampling does warm-up with 100 burn-in steps and at every new evaluation, 10 more samples are generated to approximate the posterior.

\subsection{Priors}

Especially in BO where data is scarce, priors used in the posterior sampling play an extremely important role.
The Horseshoe priors are specified for $\beta_i$'s with the design goal of variable selection as stated in the main text.
Here, we provide details about other GP-parameters including constant mean function value $m$, signal variance $\sigma_f^2$ and noise variance $\sigma_n^2$.

\subsubsection{Prior on constant mean function value $m$}

Given $\D=(V,\mathbf{y})$ the prior over the mean function is the  following:
\begin{equation}
    p(m) \propto
    \begin{cases}
        \mathcal{N}(\mu, \sigma^2) \quad \text{if} \quad y_{min} \le m \le y_{max} \\
        0 \quad \text{otherwise}
    \end{cases}
\end{equation}
where $\mu = mean(\mathbf{y})$, $\sigma = (y_{max} - y_{min}) / 4$, $y_{min} = \min(\mathbf{y})$ and $y_{max} = \max(\mathbf{y})$.

This is the truncated Gaussian distribution between $y_{min}$ and $y_{max}$ with a mean at the sample mean of $\mathbf{y}$. 
The truncation bound is set so that untruncated version can sample in truncation bound with the probability of around 0.95.

\subsubsection{Prior on signal variance $\sigma_f^2$}

Given $\D=(V,\mathbf{y})$ the prior over the log-variance is the following:
\begin{equation}
    p(\log(\sigma_f^2)) \propto
    \begin{cases}
        \mathcal{N}(\mu, \sigma^2) \quad \text{if} \quad \frac{\sigma_{\mathbf{y}}^2}{K_{VVmax}} \le \sigma_f^2 \le \frac{\sigma_{\mathbf{y}}^2}{K_{VVmin}} \\
        0 \quad \text{otherwise}
    \end{cases}
\end{equation}
where $\sigma_{\mathbf{y}}^2 = variance(\mathbf{y})$, $\mu = \frac{1}{2}(\frac{\sigma_{\mathbf{y}}^2}{K_{VVmin}} + \frac{\sigma_{\mathbf{y}}^2}{K_{VVmax}})$, $\sigma = \frac{1}{4}(\frac{\sigma_{\mathbf{y}}^2}{K_{VVmin}} + \frac{\sigma_{\mathbf{y}}^2}{K_{VVmax}})$, $K_{VVmin} = \min(K_{VV})$, $K_{VVmax} = \max(K_{VV})$ and $K_{VV} = K(V,V)$.

This is the truncated Log-Normal distribution. The intuition behind this choice of prior is that in GP prior, $\sigma_f^2 K_{VV}$ is covariance matrix of $\mathbf{y}$ with the assumption of very small noise variance $\sigma_n^2$.
The truncation bound is set so that untruncated version can sample in truncation bound with the probability of around 0.95.
Since for larger $\sigma_f^2$, the the magnitude of the change of $\sigma_f^2$ has less significant effect than for smaller $\sigma_f^2$.
In order to take into account relative amount of change in $\sigma_f^2$, the Log-Normal distribution is used rather than the Normal distribution.

\subsubsection{Priors on scaling factor $\beta_i$ and noise variance $\sigma_n^2$}

We use the Horseshoe prior for $\beta_i$ and $\sigma_n^2$ in order to encourage sparsity. Since the probability density of the Horseshoe is intractable, its closed form bound is used as a proxy \citep{carvalho2010horseshoe}:
\begin{equation}
    \frac{K}{2}\log(1+\frac{4\tau^2}{x^2}) < p(x) < K\log(1+\frac{2\tau^2}{x^2})
\end{equation}
where $x=\beta_i$ or $x=\sigma_n^2$, $\tau$ is a global shrinkage parameter and $K = (2\pi^3)^{-1/2}$~\citep{carvalho2010horseshoe}.
Typically, the upper bound is used to approximate Horseshoe density.
For $\beta_i$, we use $\tau=5$ to avoid excessive sparsity.
For $\sigma_n^2$, we use $\tau=\sqrt{0.05}$ that prefers very small noise similarly to the Spearmint implementation.\footnote{\href{url}{https://github.com/JasperSnoek/spearmint}}

\subsection{Slice Sampling}

At every new evaluation in \combo, we draw samples of $\beta_{i}$. For each $\beta_{i}$ the sampling procedure is the following:
\begin{itemize}
    \item [SS-1] Set $t=0$ and choose a starting $\beta_{i}^{(t)}$ for which the probability is non-zero.
    \item [SS-2] Sample a value $q$ uniformly from $[0, p(\beta_{i}^{(t)}|\mathcal{D},\beta_{-i}^{(t)}, m^{(t)}, (\sigma_f^2)^{(t)}, (\sigma_n^2))^{(t)}]$.
    \item [SS-3] Draw a sample $\beta_{i}$ uniformly from regions, $p(\beta_{i}|\mathcal{D},\beta_{-i}^{(t)}, m^{(t)}, (\sigma_f^2)^{(t)}, (\sigma_n^2)^{(t)}) > q$.
    \item [SS-4] Set $\beta_{i}^{(t+1)} = \beta_{i}$ and repeat from SS-2 using $\beta_{i}^{(t+1)}$.
\end{itemize}
In SS-2, we step out using doubling and shrink to draw a new value. 
For detailed explanation about slice sampling, please refer to~\citep{neal2003slice}.
For other GP-parameters, the same univariate slice sampling is applied.

\section{Acquisition function maximization}
\label{supp:sec_acquisition}
In the acquisition function maximization step, we begin with candidate vertices chosen to balance between exploration and exploitation.
$20,000$ vertices are randomly selected for exploration.
To balance exploitation, we use 20 spray vertices similar to spray points\footnote{\url{https://github.com/JasperSnoek/spearmint/blob/b37a541be1ea035f82c7c82bbd93f5b4320e7d91/spearmint/spearmint/chooser/GPEIOptChooser.py\#L235}} in~\cite{snoek2012practical}.
Spray vertices are randomly chosen in the neighborhood of a vertex with the best evaluation (e.g, $nbd(v_{best},2) = \{v \vert d(v,v_{best}) \le 2\}$).
Out of $20,020$ initial vertices, 20 vertices with the highest acquisition values are used as initial points for further optimization.
This type of combination of heuristics for exploration and exploitation has shown improved performances~\cite{garnett2010bayesian, malkomes2016bayesian}.

We use a breadth-first local search (BFLS) to further optimize the acquisition function.
At a given vertex we compare acquisition values on adjacent vertices.
We then move to the vertex with the highest acquisition value and repeat until no acquisition value on an adjacent vertex is higher than acquisition value at the current vertex.

\subsection{Non-local search for acquisition function optimization}
\label{supp:subsec_nonlocal}
We tried simulated annealing as a non-local search in different ways, namely:
\begin{enumerate}
    \setlength\itemsep{1em}
    \item Randomly split 20 initial points into 2 sets of 10 points and optimize from 10 points in one set with BFLS and optimize from 10 points in another set with simulated annealing.
    \item For given 20 initial points, optimize from 20 points with BFLS and optimize from the same 20 points with simulated annealing.
    \item For given 20 initial points, firstly optimize from 20 points with BFLS and use 20 points optimized by BFLS as initial points for optimization using simulated annealing.
\end{enumerate}
The optimum of all 3 methods is hardly better than the optimum discovered solely by BFLS. Therefore, we decided to stick to the simpler procedure without SA.

\vspace{24pt}
\section{Experiments}
\vspace{12pt}
\label{supp:sec_experiments}
\subsection{Bayesian optimization with binary variables}
\vspace{4pt}

\subsubsection{Ising sparsification}
\label{supp:exp_ising}
\vspace{4pt}

Ising sparsification is about approximating a zero-field Ising model expressed by $p(z) = \frac{1}{Z_{p}} \exp\{z^{\top} J^{p} z\}$, where $z \in \{-1,1\}^{n}$, $J^{p} \in \mathbb{R}^{n\times n}$ is an interaction symmetric matrix, and $Z_{p} = \sum_{z} \exp\{z^{\top} J^{p} z\}$ is the partition function, using a model $q(z)$ with $J_{ij}^{q} = x_{ij} J_{ij}^{p}$ where $x_{ij} \in \{0,1\}$ are the decision variables. 
The objective function is the regularized Kullback-Leibler divergence between $p$ and $q$. 
\begin{equation}
    \mathcal{L}(x) = D_{KL}(p||q) + \lambda \|x\|_{1}
\end{equation}
where $\lambda > 0$ is the regularization coefficient
$D_{KL}$ could be calculated analytically \citep{baptista2018bayesian}. We follow the same setup as presented in \citep{baptista2018bayesian}, namely, we consider $4\times 4$ grid of spins, and interactions are sampled randomly from a uniform distribution over $[0.05, 5]$. The exhaustive search requires enumerating all $2^{24}$ configurations of $x$ that is infeasible. 
We consider $\lambda \in \{0, 10^{-4}, 10^{-2}\}$. 
We set the budget to $170$ evaluations.

\begin{figure}[ht!]
    \hspace{-0pt}
    \begin{minipage}{0.78\textwidth}
        \includegraphics[width=0.75\columnwidth]{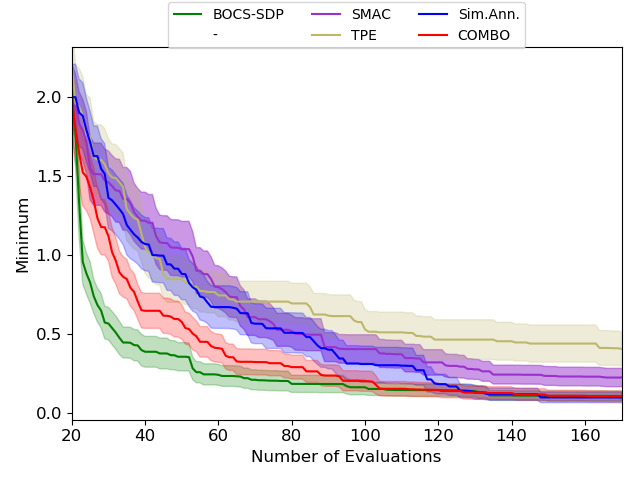}
    \end{minipage}
    \hspace{-70pt}
    \begin{minipage}{0.2\textwidth}
        \begin{tabular}{lc}
        \toprule
        Method & $\lambda=0.0$\\
        \midrule
        $\mathrm{SMAC}$     & 0.1516$\pm$0.0404\\
        $\mathrm{TPE}$      & 0.4039$\pm$0.1087\\
        $\mathrm{SA}$       & \textbf{0.0953}$\pm$0.0331\\
        $\mathrm{BOCS-SDP}$ & 0.1049$\pm$0.0308\\
        \midrule
        $\mathrm{COMBO}$    & 0.1030$\pm$0.0351\\
        \bottomrule
        \end{tabular}    
    \end{minipage}
    \vspace{-10pt}
    \caption{Ising sparsification ($\lambda=0.0$)}
    \label{supp:fig_ising+0}
\end{figure}

\begin{figure}[ht!]
    \hspace{-0pt}
    \begin{minipage}{0.78\textwidth}
        \includegraphics[width=0.75\columnwidth]{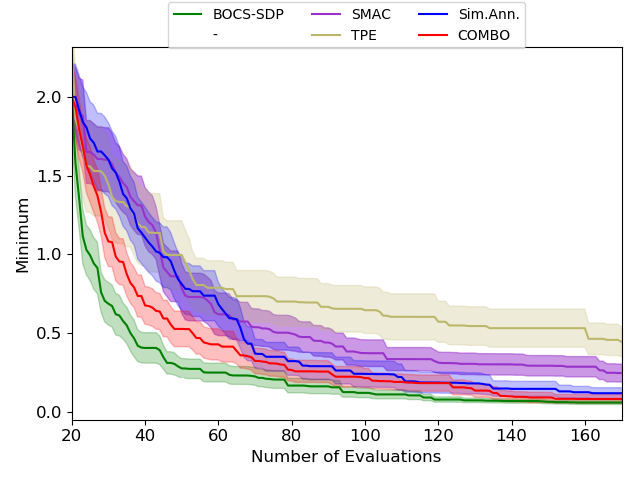}
    \end{minipage}
    \hspace{-70pt}
    \begin{minipage}{0.2\textwidth}
        \begin{tabular}{lc}
        \toprule
        Method & $\lambda=0.0001$\\
        \midrule
        $\mathrm{SMAC}$     & 0.2192$\pm$0.0522\\
        $\mathrm{TPE}$      & 0.4437$\pm$0.0952\\
        $\mathrm{SA}$       & 0.1166$\pm$0.0353\\
        $\mathrm{BOCS-SDP}$ & \textbf{0.0586}$\pm$0.0125\\
        \midrule
        $\mathrm{COMBO}$    & 0.0812$\pm$0.0279\\
        \bottomrule
        \end{tabular}    
    \end{minipage}
    \vspace{-10pt}
    \caption{Ising sparsification ($\lambda=0.0001$)}
    \label{supp:fig_ising-4}
\end{figure}

\begin{figure}[ht!]
    \hspace{-0pt}
    \begin{minipage}{0.78\textwidth}
        \includegraphics[width=0.75\columnwidth]{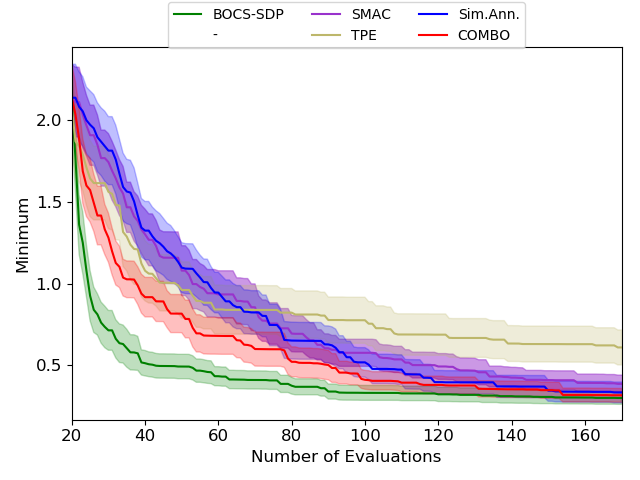}
    \end{minipage}
    \hspace{-70pt}
    \begin{minipage}{0.2\textwidth}
        \begin{tabular}{lc}
        \toprule
        Method & $\lambda=0.01$\\
        \midrule
        $\mathrm{SMAC}$     & 0.3501$\pm$0.0447\\
        $\mathrm{TPE}$      & 0.6091$\pm$0.1071\\
        $\mathrm{SA}$       & 0.3342$\pm$0.0636\\
        $\mathrm{BOCS-SDP}$ & \textbf{0.3001}$\pm$0.0389\\
        \midrule
        $\mathrm{COMBO}$    & 0.3166$\pm$0.0420\\
        \bottomrule
        \end{tabular}    
    \end{minipage}
    \vspace{-10pt}
    \caption{Ising sparsification ($\lambda=0.01$)}
    \label{supp:fig_ising-2}
\end{figure}

\subsubsection{Contamination control}
\label{supp:exp_contamination}

The contamination control in food supply chain is a binary optimization problem \citep{hu2010contamination}.
The problem is about minimizing the contamination of food where at each stage a prevention effort can be made to decrease a possible contamination.
Applying the prevention effort results in an additional cost $c_i$.
However, if the food chain is contaminated at stage $i$, the contamination spreads at rate $\alpha_i$.
The contamination at the $i$-th stage is represented by a random variable $\Gamma_i$.
A random variable $z_i$ denotes a fraction of contaminated food at the $i$-th stage, and it could be expressed in an recursive manner, namely, $z_{i} = \alpha_i (1 - x_{i})(1 - z_{i-1}) + (1 - \Gamma_i x_i)z_{i-1}$, where $x_i \in \{0, 1\}$ is the decision variable representing the preventing effort at stage $i$.
Hence, the optimization problem is to make a decision at each stage whether the prevention effort should be applied so that to minimize the general cost while also ensuring that the upper limit of contamination is $u_i$ with probability at least $1-\varepsilon$.
The initial contamination and other random variables follow beta distributions that results in the following objective function
\begin{equation}
    \mathcal{L}(x) = \sum_{i=1}^{d}\Big{[} c_{i} x_{i} + \frac{\rho}{T} \sum_{k=1}^{T} 1_{\{z_{k}>u_i\}} \Big{]} + \lambda \| x \|_{1}   
\end{equation}
where $\lambda$ is a regularization coefficient, $\rho$ is a penalty coefficient (we use $\rho = 1$) and we set $T=100$. Following \citep{baptista2018bayesian}, we assume $u_i = 0.1$, $\varepsilon=0.05$, and $\lambda \in \{0, 10^{-4}, 10^{-2}\}$. We set the budget to $270$ evaluations.

\begin{figure}[ht!]
    \hspace{-0pt}
    \begin{minipage}{0.78\textwidth}
        \includegraphics[width=0.75\columnwidth]{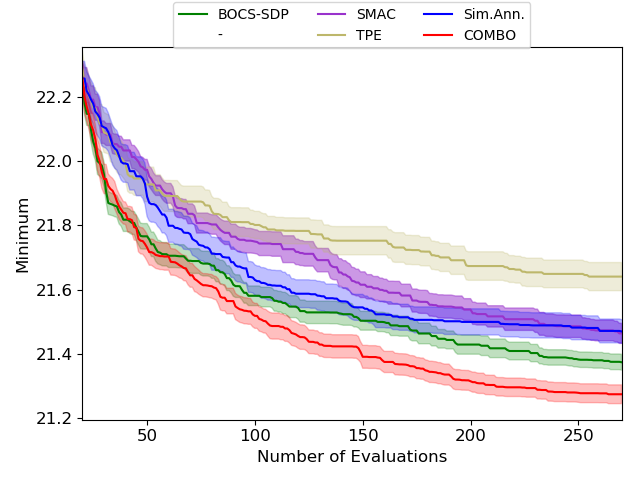}
    \end{minipage}
    \hspace{-70pt}
    \begin{minipage}{0.2\textwidth}
        \begin{tabular}{lc}
        \toprule
        Method & $\lambda=0.0$\\
        \midrule
        $\mathrm{SMAC}$     & 21.4644$\pm$0.0312\\
        $\mathrm{TPE}$      & 21.6408$\pm$0.0437\\
        $\mathrm{SA}$       & 21.4704$\pm$0.0366\\
        $\mathrm{BOCS-SDP}$ & 21.3748$\pm$0.0246\\
        \midrule
        $\mathrm{COMBO}$    & \textbf{21.2752}$\pm$0.0292\\
        \bottomrule
        \end{tabular}    
    \end{minipage}
    \vspace{-10pt}
    \caption{Contamination control ($\lambda=0.0$)}
    \label{supp:fig_cont+0}
\end{figure}

\begin{figure}[ht!]
    \hspace{-0pt}
    \begin{minipage}{0.78\textwidth}
        \includegraphics[width=0.75\columnwidth]{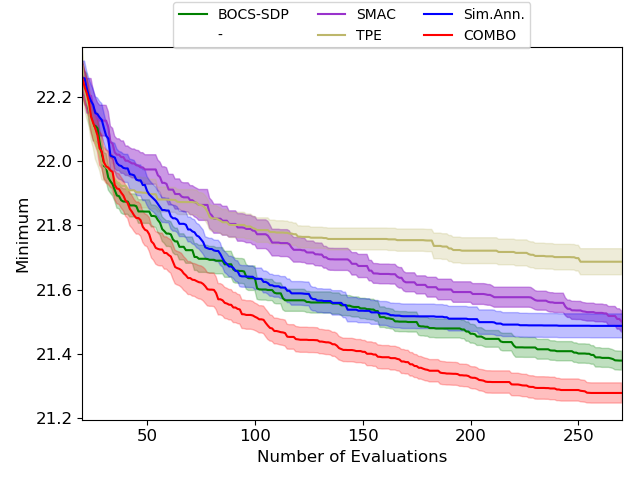}
    \end{minipage}
    \hspace{-70pt}
    \begin{minipage}{0.2\textwidth}
        \begin{tabular}{lc}
        \toprule
        Method & $\lambda=0.0001$\\
        \midrule
        $\mathrm{SMAC}$     & 21.5011$\pm$0.0329\\
        $\mathrm{TPE}$      & 21.6868$\pm$0.0406\\
        $\mathrm{SA}$       & 21.4871$\pm$0.0372\\
        $\mathrm{BOCS-SDP}$ & 21.3792$\pm$0.0296\\
        \midrule
        $\mathrm{COMBO}$    & \textbf{21.2784}$\pm$0.0314\\
        \bottomrule
        \end{tabular}    
    \end{minipage}
    \vspace{-10pt}
    \caption{Contamination control ($\lambda=0.0001$)}
    \label{supp:fig_cont-4}
\end{figure}

\begin{figure}[ht!]
    \hspace{-0pt}
    \begin{minipage}{0.78\textwidth}
        \includegraphics[width=0.75\columnwidth]{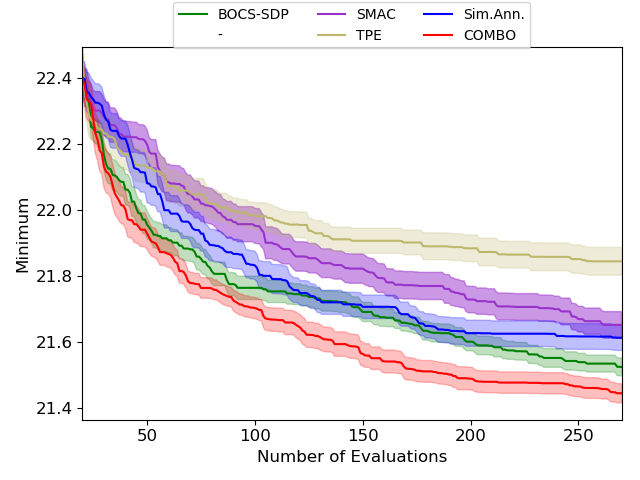}
    \end{minipage}
    \hspace{-70pt}
    \begin{minipage}{0.2\textwidth}
        \begin{tabular}{lc}
        \toprule
        Method & $\lambda=0.01$\\
        \midrule
        $\mathrm{SMAC}$     & 21.6512$\pm$0.0403\\
        $\mathrm{TPE}$      & 21.8440$\pm$0.0422\\
        $\mathrm{SA}$       & 21.6120$\pm$0.0385\\
        $\mathrm{BOCS-SDP}$ & 21.5232$\pm$0.0269\\
        \midrule
        $\mathrm{COMBO}$    & \textbf{21.4436}$\pm$0.0293\\
        \bottomrule
        \end{tabular}    
    \end{minipage}
    \vspace{-10pt}
    \caption{Contamination control ($\lambda=0.01$)}
    \label{supp:fig_cont-2}
\end{figure}

\newpage
\subsection{Bayesian optimization with ordinal and multi-categorical variables}

\subsubsection{Oridinal variables : discretized branin}
\label{supp:exp_branin}
In order to test \combo on ordinal variables. We adopt widely used continuous benchmark branin function.
Branin is defined on $[0,1]^2$, we discretize each dimension with 51 equally space points so that center point can be chosen in the discretized space.
Therefore, the search space is comprised of 2 ordinal variables with 51 values(choices) for each.

\combo outperforms all of its competitors. 
In Figure~\ref{supp:fig_branin}, SMAC and TPE exhibit similar search progress as \combo, but in term of the final value at 100 evaluations, those two are overtaken by SA.
\combo maintains its better performance over all range of evaluations up to 100.

\begin{figure}[ht!]
    \hspace{-0pt}
    \begin{minipage}{0.78\textwidth}
        \includegraphics[width=0.75\columnwidth]{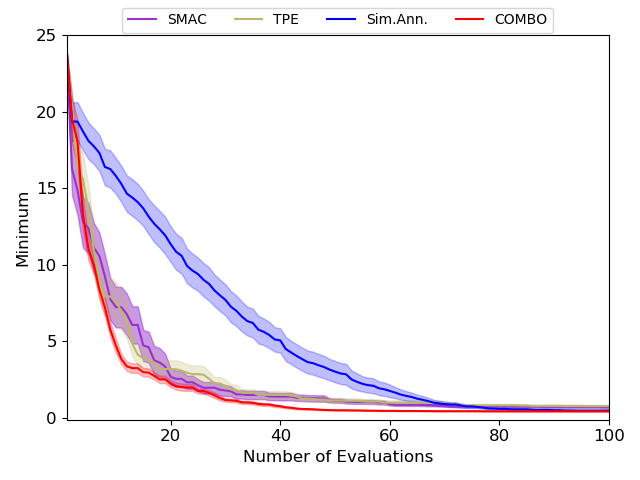}
    \end{minipage}
    \hspace{-70pt}
    \begin{minipage}{0.2\textwidth}
        \begin{tabular}{lc}
        \toprule
        Method & Branin\\
        \midrule
        $\mathrm{SMAC}$  & 0.6962$\pm$0.0705\\
        $\mathrm{TPE}$   & 0.7578$\pm$0.0844\\
        $\mathrm{SA}$    & 0.4659$\pm$0.0211\\
        \midrule
        $\mathrm{COMBO}$ & \textbf{0.4113}$\pm$0.0012\\
        \bottomrule
        \end{tabular}    
    \end{minipage}
    \vspace{-10pt}
    \caption{Ordinal variables : discretized branin}
    \label{supp:fig_branin}
\end{figure}

\subsubsection{Multi-categorical variables : pest control}
\label{supp:exp_pestcontrol}
\vspace{-4pt}

\begin{figure}[ht!]
    \hspace{-0pt}
    \begin{minipage}{0.78\textwidth}
        \includegraphics[width=0.75\columnwidth]{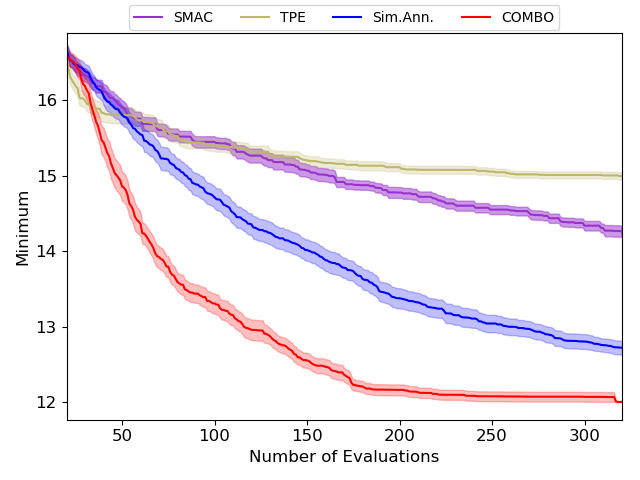}
    \end{minipage}
    \hspace{-70pt}
    \begin{minipage}{0.2\textwidth}
        \begin{tabular}{lc}
        \toprule
        Method & Pest\\
        \midrule
        $\mathrm{SMAC}$     & 14.2614$\pm$0.0753\\
        $\mathrm{TPE}$      & 14.9776$\pm$0.0446\\
        $\mathrm{SA}$       & 12.7154$\pm$0.0918\\
        \midrule
        $\mathrm{COMBO}$    & \textbf{12.0012}$\pm$0.0033\\
        \bottomrule
        \end{tabular}    
    \end{minipage}
    \vspace{-10pt}
    \caption{Multi-categorical variables : pest control}
    \label{supp:fig_pest}
\end{figure}

In the chain of stations, pest is spread in one direction, at each pest control station, the pest control officer can choose to use a pesticide from 4 different companies which differ in their price and effectiveness.

For $N$ pest control stations, the search space for this problem is $5^N$, 4 choices of a pesticide and the choice of not using any of it.

The price and effectiveness reflect following dynamics.
\begin{itemize}
    \setlength\itemsep{0.1em}
    \item If you have purchased a pesticide a lot, then in your next purchase of the same pesticide, you will get discounted proportional to the amount you have purchased.
    \item If you have used a pesticide a lot, then pests will acquire strong tolerance to that specific product, which decrease effectiveness of that pesticide.
\end{itemize}

Formally, there are four variables: 
at $i$-th pest control
$Z_i$ is the portion of the product having pest,
$A_i$ is the action taken,
$C_i^{(l)}$ is the adjusted cost of pesticide of type $l$,
$T_i^{(l)}$ is the beta parameter of the Beta distribution for the effectiveness of pesticide of type $l$.
It starts with initial $Z_0$ and follows the same evolution as in the contamination control, but after each choice of pesticide type whenever the taken action is to use one out of 4 pesticides or no action.
$\{C_i^{(l)}\}_{1, \cdots, 4}$ are adjusted in the manner that the pesticide which has been purchased most often will get a discount for the price.
$\{T_i^{(l)}\}_{1, \cdots, 4}$ are adjusted in the fashion that the pesticide which has been frequently used in previous control point cannot be as effective as before since the insects have developed tolerance to that.

The portion of the product having pest follows the dynamics below
\begin{equation}
    z_{i} = \alpha_i (1 - x_{i})(1 - z_{i-1}) + (1 - \Gamma_i x_i)z_{i-1}
\end{equation}
when the pesticide is used, the effectiveness $x_{i}$ of pesticide follows beta distribution with the parameters, which has been adjusted according to the sequence of actions taken in previous control points.

Under this setting, our goal is to minimize the expense for pesticide control and the portion of products having pest while going through the chain of pest control stations.
The objective is similar to the contamination control problem
\begin{equation}
    \mathcal{L}(x) = \sum_{i=1}^{d}\Big{[} c_{i} x_{i} + \frac{\rho}{T} \sum_{k=1}^{T} 1_{\{z_{k}>u_i\}} \Big{]}
\end{equation}

However, we want to stress out that the dynamics of this problem is far more complex than the one in the contamination control case. 
First, it has $25$ variables and each variable has $5$ categories.
More importantly, the price and effectiveness of pesticides are dynamically adjusted depending on the previously made choice.

\subsection{Weighted maximum satisfiability(wMaxSAT)}
\label{supp:exp_wmaxsat}
\vspace{-4pt}
Satisfiability problem is the one of the most important and general form of combinatorial optimization problems. 
SAT solver competition is held in Satisfiability conference every year.\footnote{\href{url}{http://satisfiability.org/}, \href{url}{http://sat2018.azurewebsites.net/competitions/}}
Due to the resemblance between combinatorial optimizations and weighted Maximum satisfiability problems, in which the goal is to find boolean values that maximize the combined weights of satisfied clauses, we optimize certain benchmarks taken from Maximum atisfiability(MaxSAT) Competition 2018.
We took randomly 3 benchmarks of weighted maximum satisfiability problems with no hard clause with the number of variables not exceeding 100.\footnote{\href{url}{https://maxsat-evaluations.github.io/2018/benchmarks.html} maxcut-johnson8-2-4.clq.wcnf (28 variables), maxcut-hamming8-2.clq.wcnf (43 variables), frb-frb10-6-4.wcnf (60 variables)}
The weights are normalized by mean subtraction and standard deviation division and evaluations are negated to be minimization problems.

\begin{figure}[!h]
    \hspace{-0pt}
    \begin{minipage}{0.78\textwidth}
        \includegraphics[width=0.75\columnwidth]{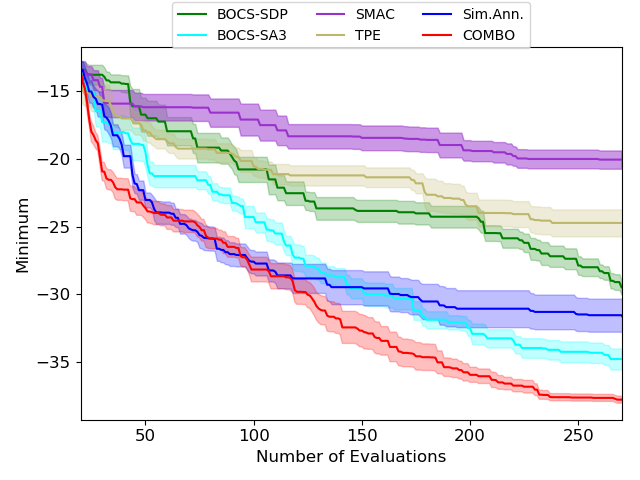}
    \end{minipage}
    \hspace{-70pt}
    \begin{minipage}{0.2\textwidth}
        \begin{tabular}{lc}
        \toprule
        Method & 28\\
        \midrule
        $\mathrm{SMAC}$            & -20.0530$\pm$0.6735\\
        $\mathrm{TPE}$             & -25.2010$\pm$0.8750\\
        $\mathrm{SA}$              & -31.8060$\pm$1.1929\\
        $\mathrm{BOCS\text{-}SDP}$ & -29.4865$\pm$0.5348\\
        $\mathrm{BOCS\text{-}SA3}$ & -34.7915$\pm$0.7814\\
        \midrule
        $\mathrm{COMBO}$           & \textbf{-37.7960}$\pm$0.2662\\
        \bottomrule
        \end{tabular}    
    \end{minipage}
    \vspace{-10pt}
    \caption{wMaxSAT28}
    \label{supp:fig_wmaxsat28}
    \vspace{-8pt}
\end{figure}

\begin{figure}[!ht]
    \hspace{-0pt}
    \begin{minipage}{0.78\textwidth}
        \includegraphics[width=0.75\columnwidth]{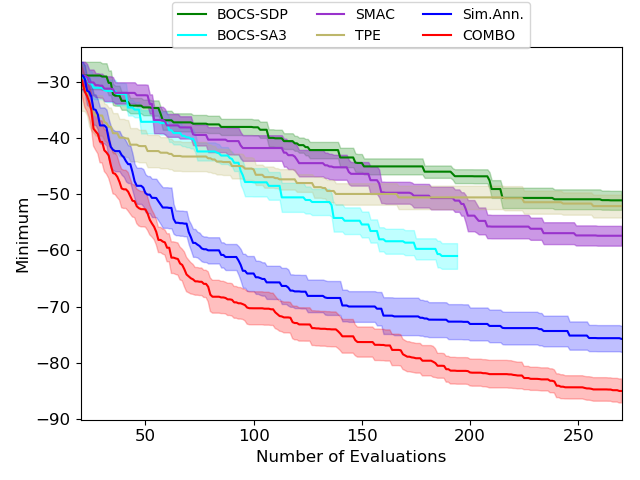}
    \end{minipage}
    \hspace{-70pt}
    \begin{minipage}{0.2\textwidth}
        \begin{tabular}{lc}
        \toprule
        Method & 43\\
        \midrule
        $\mathrm{SMAC}$              & -57.4217$\pm$1.7614\\
        $\mathrm{TPE}$               & -52.3856$\pm$1.9898\\
        $\mathrm{SA}$                & -75.7582$\pm$2.3048\\
        $\mathrm{BOCS\text{-}SDP}$   & -51.1265$\pm$1.6903\\
        $\mathrm{BOCS\text{-}SA3}^*$ & -61.0186$\pm$2.2812\\
        \midrule
        $\mathrm{COMBO}$             & \textbf{-85.0155}$\pm$2.1390\\
        \bottomrule
        \end{tabular}    
    \end{minipage}
    \vspace{-10pt}
    \caption[caption]{wMaxSAT43 \\ \footnotesize{$^*$BOCS-SA3 was run for 168 hours but could not finish 270 evaulations.}}
    \label{supp:fig_wmaxsat43}
\end{figure}

\begin{figure}[!ht]
    \hspace{-0pt}
    \begin{minipage}{0.78\textwidth}
        \includegraphics[width=0.75\columnwidth]{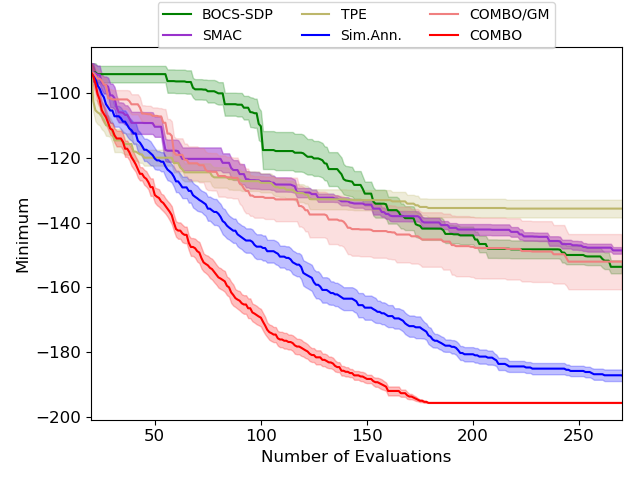}
    \end{minipage}
    \hspace{-70pt}
    \begin{minipage}{0.2\textwidth}
        \begin{tabular}{lc}
        \toprule
        Method & 60\\
        \midrule
        $\mathrm{SMAC}$            & -148.6020$\pm$1.0135\\
        $\mathrm{TPE}$             & -137.2104$\pm$2.8296\\
        $\mathrm{SA}$              & -187.5506$\pm$1.4962\\
        $\mathrm{BOCS\text{-}SDP}$ & -153.6722$\pm$2.0096\\
        $\mathrm{COMBO/GM}$        & -152.0745$\pm$8.5167\\
        \midrule
        $\mathrm{COMBO}$           & \textbf{-195.6527}$\pm$0.0000\\
        \bottomrule
        \end{tabular}    
    \end{minipage}
    \vspace{-10pt}
    \caption{wMaxSAT60}
    \label{supp:fig_wmaxsat60}
\end{figure}

\begin{figure}[!ht]
    \centering
    \includegraphics[width=0.75\columnwidth]{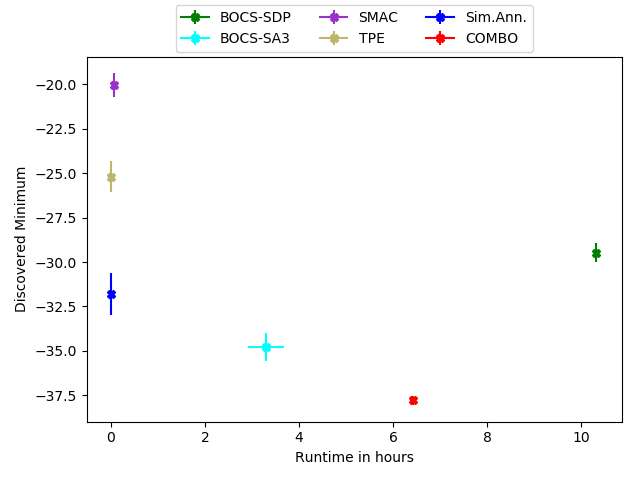}
    \caption{Runtime VS. Minimum on wMaxSAT28}
    \label{supp:fig_wmaxsat28_runtime}
\end{figure}

\begin{figure}[!ht]
    \centering
    \includegraphics[width=0.75\columnwidth]{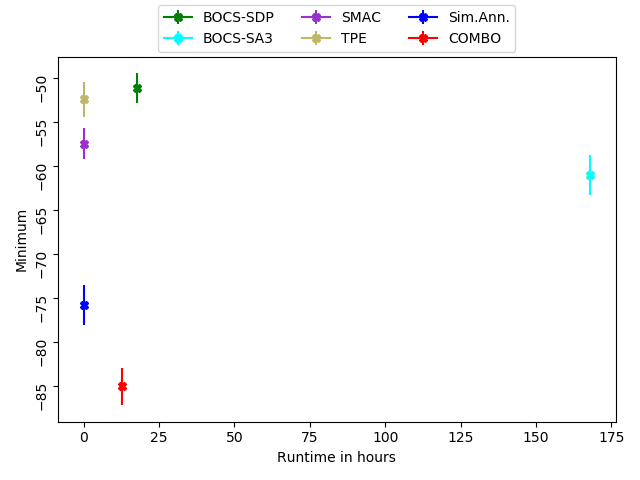}
    \caption{Runtime VS. Minimum on wMaxSAT4. BOCS-SA3 did not finish all 270 evaluations after 168 hours, we plot the runtime for BOCS-SA3 as 168 hours.}
    \label{supp:fig_wmaxsat43_runtime}
\end{figure}

\newpage
\subsection{Neural architecture search(NAS)}
\label{supp:exp_nas}

\subsubsection{Search space}
\label{supp:subsec_nas_search_space}

\begin{table}[!h]
    \caption{(\textit{left}) Connectivity and (\textit{right)} Computation type.}
    \begin{minipage}{0.5\textwidth}
    \begin{center}
    \begin{scriptsize}
    \begin{sc}
    \begin{tabular}{c|ccccccc}
        & \textbf{IN} & \textbf{H1} & \textbf{H2} & \textbf{H3} & \textbf{H4} & \textbf{H5} & \textbf{OUT}\\
        \midrule
        \textbf{IN} & - & O & X & X & X & O & X \\
        \textbf{H1} & - & - & X & O & X & X & O \\
        \textbf{H2} & - & - & - & X & O & X & X \\
        \textbf{H3} & - & - & - & - & X & O & X \\
        \textbf{H4} & - & - & - & - & - & O & O \\
        \textbf{H5} & - & - & - & - & - & - & X \\
        \textbf{OUT}& - & - & - & - & - & - & - \\
    \end{tabular}  
    \end{sc}
    \end{scriptsize}
    \end{center}
    \end{minipage}
    \hspace{2pt}
    \begin{minipage}{0.5\textwidth}
    \begin{scriptsize}
    \begin{sc}
    \begin{tabular}{c|c|c}
        \toprule
            & \textbf{MaxPool} & \textbf{Conv}\\
        \midrule
        \textbf{Small}  & Id $\equiv$ MaxPool(1$\times$1)& Conv(3$\times$3)\\
        \midrule
        \textbf{Large}  & MaxPool(3$\times$3)            & Conv(5$\times$5)\\
        \bottomrule
    \end{tabular}    
    \end{sc}
    \end{scriptsize}
    \end{minipage}
    \label{supp:tbl_nas_search_space}
\end{table}

In our architecture search problem, the cell consists of one input state(\textbf{IN}), one output state(\textbf{OUT}) and five hidden states(\textbf{H1}$\sim$\textbf{H5}).
The connectivity between 7 states are specified as in the left of Table.~\ref{supp:tbl_nas_search_space} where it can be read that (\textbf{IN}$\rightarrow$\textbf{H1}) and (\textbf{IN}$\rightarrow$\textbf{H5}) from the first row.
Input and output states are identity maps. 
The computation type of each of 5 hidden states are determined by combination of 2 binary choices as in the right of Table.~\ref{supp:tbl_nas_search_space}.

In total, our search space consists of 31 binary variables.\footnote{We design a binary search space for NAS so that to also compare with  BOCS. \combo is not restricted to binary choices for NAS, however.}

\subsubsection{Evaluation}

For a given 31 binary choices, we construct a cell and stack 3 cells as follows
\begin{center}
    Input\\
    $\downarrow$\\
    Conv($3 \times 16 \times 3 \times 3$)-BN-ReLU\\
    $\downarrow$\\
    Cell with 16 channels\\
    $\downarrow$\\
    MaxPool($2\times2$)-Conv($16 \times 32 \times 1 \times 1$)\\
    $\downarrow$\\
    Cell with 32 channels\\
    $\downarrow$\\
    MaxPool($2\times2$)-Conv($32 \times 64 \times 1 \times 1$)\\
    $\downarrow$\\
    Cell with 64 channels\\
    $\downarrow$\\
    MaxPool($2\times2$)-FC($1024 \times 10$)\\
    $\downarrow$\\
    Output
\end{center}
At each MaxPool, the height and the width of features are halved.

The network is trained for 20 epochs with Adam~\citep{kingma2014adam} with the default settings in pytorch~\citep{paszke2017automatic} except for the weight decay of $5\times10^{-5}$.
CIFAR10~\citep{krizhevsky2009learning} training data is randomly shuffled with random seed 0 in the command ``numpy.RandomState(0).shuffle(indices)''.
In the shuffled training data, the first 30000 is used for training and the last 10000 is used for evaluations.
Batch size is 100.
Early stopping is applied when validation accuracy is worse than the validation accuracy 10 epochs ago.

Due to the small number of epochs, evaluations are a bit noisy. 
In order to stabilize evaluations, we run 4 times of training for a given architecture.
On GTX 1080 Ti with 11GB, 4 runs can be run in parallel.
Depending on a given architecture training took approximately 5$\sim$30 minutes.

Since the some binary choices result in invalid architectures, in such case, validation accuracy is given as 10\%, which is the expected accuracy of constant prediction.

The final evaluation is given as
\begin{equation}
    Error_{validation} + 0.02\cdot\frac{\text{FLOPs of a given architecture}}{\text{Maximim FLOPs in the search space}}
\end{equation}
where ``Maximim FLOPs in the search space'' is computed from the cell with all connectivity among states and Conv($5\times5$) for all \textbf{H1}$\sim$\textbf{H5}.
$0.02$ is set with the assumption that we can afford $1\%$ of increased error with $50\%$ reduction in FLOPs.

\begin{figure}[h!]
    \hspace{-0pt}
    \begin{minipage}{0.78\textwidth}
        \includegraphics[width=0.75\columnwidth]{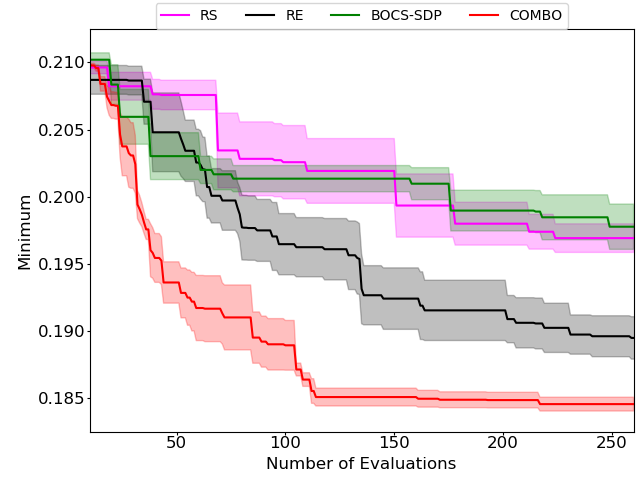}
    \end{minipage}
    \hspace{-70pt}
    \begin{minipage}{0.2\textwidth}
        \begin{tabular}{lc}
        \toprule
        Method & NAS\\
        \midrule
        $\mathrm{RS}$       & 0.1969$\pm$0.0011\\
        $\mathrm{BOCS-SDP}$ & 0.1978$\pm$0.0017\\
        $\mathrm{RE}$       & 0.1895$\pm$0.0016\\
        \midrule
        $\mathrm{COMBO}$    & \textbf{0.1846}$\pm$0.0005\\
        \bottomrule
        \end{tabular}    
    \end{minipage}
    \vspace{-6pt}
    \caption{Neural architecture search experiment.}
    \label{supp:fig_nas}
\end{figure}

\begin{figure}[h!]
    \hspace{-0pt}
    \begin{minipage}{0.78\textwidth}
        \includegraphics[width=0.75\columnwidth]{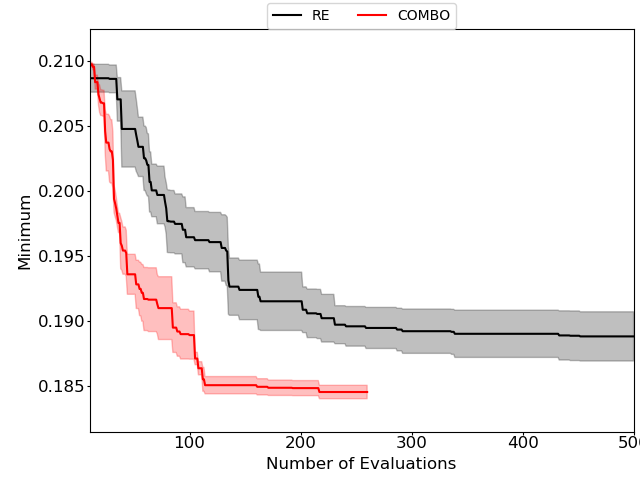}
    \end{minipage}
    \hspace{-70pt}
    \begin{minipage}{0.2\textwidth}
        \begin{tabular}{lc}
        \toprule
        Method(\#eval) & NAS\\
        \midrule
        $\mathrm{RE}(260)$       & 0.1895$\pm$0.0016\\
        $\mathrm{RE}(500)$       & 0.1888$\pm$0.0019\\
        \midrule
        $\mathrm{COMBO}(260)$    & \textbf{0.1846}$\pm$0.0005\\
        \bottomrule
        \end{tabular}    
    \end{minipage}
    \vspace{-6pt}
    \caption{Neural architecture search experiment with additional evaluations for RE (up to 500 evaluations).}
    \label{supp:fig_nas_500}
\end{figure}

\subsubsection{Comparison to NASNet search space}

\begin{center}
    \begin{tabular}{c|c|c}
        \toprule
        Binary &  & NASNet\\
        \midrule
        Yes & Invalid Architecture & No\\
        Not fixed & The number of inputs to each state & 2\\
        4 & The number of computation type of states & 13\\
        \bottomrule
    \end{tabular}  
\end{center}

\subsubsection{Regularized evolution hyperparameters}
In evolutionary search algorithms, the choice of mutation is critical to the performance.
Since our binary search space is different from NASNet search space where Regulairzed Evolution(RE) was originally applied, we modify mutation steps.
All possible mutations proposed in~\citep{real2018regularized} can be represented as simple binary flipping in binary search space.
In binary search space, some binary variables are about computation type and others are about connectivity.
Thus uniform-randomly choosing binary variable to flip can mutate computation type or connectivity.
Since binary search space is redundant we did not explicitly include identity mutation (not mutating).
Since evolutionary search algorithms are believed to be less sample efficient than BO, we gave an advantage to RE by only allowing valid architectures in mutation steps.

On other crucial hyperparameters, population size $P$ and sample size $S$, motivated by the best value used in~\citep{real2018regularized}, $P=100$, $S=25$.
We set our $P$ and $S$ to have similar ratio as the one originally proposed.
Since, we assumed less number of evaluations(260, 500) compared to 20000 in~\citep{real2018regularized}, we reduced $P$ and $S$.
In NAS on binary search space, we used $P=50$ and $S=15$.

\end{document}